\documentclass[colorlinks]{article}
\usepackage[T1]{fontenc}
\usepackage[utf8]{inputenc}
\usepackage{color}
\usepackage{float}
\usepackage{bm}
\usepackage{amsmath}
\usepackage{amsthm}
\usepackage{graphicx}
\usepackage[numbers]{natbib}
\usepackage{microtype}
\usepackage[unicode=true,
 bookmarks=true,bookmarksnumbered=true,bookmarksopen=true,bookmarksopenlevel=1,
 breaklinks=false,pdfborder={0 0 0},pdfborderstyle={},backref=false,colorlinks=true]
 {hyperref}

\makeatletter

\floatstyle{ruled}
\newfloat{algorithm}{tbp}{loa}
\providecommand{\algorithmname}{Algorithm}
\floatname{algorithm}{\protect\algorithmname}

\theoremstyle{plain}
\newtheorem{thm}{\protect\theoremname}
\theoremstyle{plain}
\newtheorem{cor}[thm]{\protect\corollaryname}
\theoremstyle{remark}
\newtheorem{claim}[thm]{\protect\claimname}

\usepackage{xcolor}
\definecolor{mycol}{rgb}{0,0,0.65}
\hypersetup{
    colorlinks,
    linkcolor={mycol},
    citecolor={mycol},
    urlcolor={mycol}
}

\usepackage{amsfonts,bm,amsmath,amssymb}
\usepackage{lipsum}
\usepackage{nicefrac}

\usepackage[final]{neurips_2020}

\usepackage{url}
\usepackage{amsfonts}
\usepackage{nicefrac}

\author{Justin Domke}

\newcommand{\renewtheorem}[1]{%
  \expandafter\let\csname #1\endcsname\relax
  \expandafter\let\csname c@#1\endcsname\relax
  \expandafter\let\csname end#1\endcsname\relax
  \newtheorem{#1}%
}

\DeclareMathAlphabet{\mathbfsf}{\encodingdefault}{\sfdefault}{bx}{n}
\newcommand{\upgreektemplate}[2]{#2{
\renewcommand{\alpha}{\upalpha}
\renewcommand{\beta}{\upbeta}
\renewcommand{\theta}{\uptheta}
\renewcommand{\gamma}{\upgamma}
\renewcommand{\lambda}{\uplambda}
\renewcommand{\xi}{\upxi}
\renewcommand{\epsilon}{\upepsilon}
\renewcommand{\delta}{\updelta}
\renewcommand{\phi}{\upphi}
\renewcommand{\zeta}{\upzeta}
\renewcommand{\Lambda}{\Uplambda}
\renewcommand{\Gamma}{\Upgamma}
\renewcommand{\Delta}{\Updelta}
\renewcommand{\Theta}{\Uptheta}
#1
}}
\newcommand{\upgreek}[1]{\upgreektemplate{#1}{\mathsf}}
\newcommand{\bupgreek}[1]{\upgreektemplate{#1}{\mathbfsf}}
\usepackage{upgreek}

\usepackage{thmtools}
\usepackage{thm-restate}
\usepackage{hyperref}

\usepackage{enumitem}
\setlist[itemize]{leftmargin=*}
\setenumerate{itemsep=-2pt,topsep=-5pt,partopsep=-1pt,parsep=-1pt}

\usepackage[nameinlink,capitalize]{cleveref}
\AtBeginDocument{%
\let\ref\cref
}
\Crefname{equation}{Eq.}{Eqs.}
\Crefname{section}{Sec.}{Secs.}
\Crefname{algorithm}{Alg.}{Algs.}
\Crefname{thm}{Thm.}{Thms.}
\Crefname{cor}{Cor.}{Cors.}

\theoremstyle{theorem}
\renewtheorem{cor}[thm]{Corollary}
\renewtheorem{claim}[thm]{Claim}

\makeatother

\providecommand{\claimname}{Claim}
\providecommand{\corollaryname}{Corollary}
\providecommand{\theoremname}{Theorem}

\begin{document}
\global\long\def\argmin{\operatornamewithlimits{argmin}}%

\global\long\def\argmax{\operatornamewithlimits{argmax}}%

\global\long\def\prox{\operatornamewithlimits{prox}}%

\global\long\def\diag{\operatorname{diag}}%

\global\long\def\lse{\operatorname{lse}}%

\global\long\def\R{\mathbb{R}}%

\global\long\def\E{\operatornamewithlimits{\mathbb{E}}}%

\global\long\def\P{\operatornamewithlimits{\mathbb{P}}}%

\global\long\def\V{\operatornamewithlimits{\mathbb{V}}}%

\global\long\def\N{\mathcal{N}}%

\global\long\def\L{\mathcal{L}}%

\global\long\def\C{\mathbb{C}}%

\global\long\def\tr{\operatorname{tr}}%

\global\long\def\norm#1{\left\Vert #1\right\Vert }%

\global\long\def\norms#1{\left\Vert #1\right\Vert ^{2}}%

\global\long\def\pars#1{\left(#1\right)}%

\global\long\def\pp#1{(#1)}%

\global\long\def\bracs#1{\left[#1\right]}%

\global\long\def\bb#1{[#1]}%

\global\long\def\verts#1{\left\vert #1\right\vert }%

\global\long\def\vv#1{\vert#1\vert}%

\global\long\def\Verts#1{\left\Vert #1\right\Vert }%

\global\long\def\VV#1{\Vert#1\Vert}%

\global\long\def\angs#1{\left\langle #1\right\rangle }%

\global\long\def\KL#1{[#1]}%

\global\long\def\KL#1#2{{\scriptstyle \operatorname{KL}}\hspace{-2pt}\pars{#1\middle\Vert#2}}%

\global\long\def\SKL#1#2{{\scriptstyle \operatorname{SKL}}\hspace{-2pt}\pars{#1\middle\Vert#2}}%

\global\long\def\mean{{\displaystyle \operatorname{ave}}}%

\global\long\def\div{\text{div}}%

\global\long\def\erf{\text{erf}}%

\global\long\def\vvec{\text{vec}}%

\global\long\def\b#1{\bm{#1}}%

\global\long\def\r#1{\upgreek{#1}}%

\global\long\def\br#1{\bupgreek{\bm{#1}}}%

\global\long\def\w{\b w}%

\global\long\def\v{\b v}%

\global\long\def\wr{\br w}%

\global\long\def\z{\b z}%

\global\long\def\y{\b y}%

\global\long\def\yr{\r y}%

\global\long\def\x{\b x}%

\global\long\def\xr{\r x}%

\global\long\def\zr{\r z}%

\global\long\def\h{\b h}%

\global\long\def\hr{\r h}%

\global\long\def\u{\b u}%

\global\long\def\ur{\r u}%

\global\long\def\gr{\r g}%

\global\long\def\hr{\r h}%

\title{An Easy to Interpret Diagnostic for Approximate Inference: Symmetric
Divergence Over Simulations}

\maketitle
\maketitle
\begin{abstract}
It is important to estimate the errors of probabilistic inference
algorithms. Existing diagnostics for Markov chain Monte Carlo methods
assume inference is asymptotically exact, and are not appropriate
for approximate methods like variational inference or Laplace's method.
This paper introduces a diagnostic based on repeatedly simulating
datasets from the prior and performing inference on each. The central
observation is that it is possible to estimate a symmetric KL-divergence
defined over these simulations.

\end{abstract}

\section{Introduction\label{sec:Background}}

This paper considers the probabilistic inference problem. Given a
known distribution $p\pp{\z,\x}$ and observing some specific value
$\x$, one wishes to infer $\z$. (E.g. predict its mean or variance.)
Unless $p$ is simple, there is no simple form for the posterior $p\pp{\z\vert\x}$.
Many approximate methods exist, including variants of MCMC, message-passing,
Laplace's method, and variational inference (VI). All these methods
produce large errors on some problems. Thus, diagnostic techniques
are of high interest to understand when a given inference method will
perform well on a given problem.

For MCMC, there are several widely-used diagnostics. The potential
scale reduction factor $\hat{R}$ diagnostic \citep{Gelman_1992_InferenceIterativeSimulation}
runs multiple chains, and then compares within-chain and between-chain
variances. The expected sample size diagnostic considers correlations
in a single chain. Diagnostics of this type are an active research
area \citep{Vehtari_2020_Ranknormalizationfoldinglocalization}.

While successful, these diagnostics are grounded in the fact that
MCMC is asymptotically exact. That is, under mild conditions MCMC
will converge to the stationary distribution if run long enough. Informally,
diagnostics for MCMC only need to diagnose ``\emph{has the chain
converged?}'', rather than ``\emph{has it converged to the correct
distribution?}''.

For inference methods that are asymptotically approximate, different
diagnostics are needed. This paper is in the line of \emph{simulation-based}
diagnostics. These are a fairly radical departure. Rather than measuring
how well inference performs on the given $\x$, these estimate how
well inference performs \emph{on average} over data generated from
the model. These diagnostics repeatedly sample $\pp{\zr,\xr}\sim p\pp{\z,\x}$
and then do inference on the simulated $\xr$. The power of this approach
is that the true latent $\zr$ corresponding to the observed $\r x$
is known. 

To the best of our knowledge, this simulation-based approach was first
pursued by \citet{Cook_2006_ValidationSoftwareBayesian}, who sample
$\xr\sim p\pp{\x}$ and then perform inference to approximately sample
$\zr\sim p\pp{\z\vert\xr}.$ The quantiles of each component $\zr_{i}$
generated this way are compared to those generated directly from the
prior $p\pp{\z}.$ This can be done visually (looking at histograms),
or by using a Kolmogorov-Smirnov test. More recently, \citet{Yao_2018_YesDidIt}
suggest testing for symmetry. These error measures may not be appropriate
for all situations. First, these measures can be challenging to automate
or interpret, since they do not provide a scalar quantity but rather
a procedure to perform in each dimension. Second, there could be inference
errors not detected by looking at univariate distributions.

In this paper, we observe that some inference methods, such as Laplace's
method and variational inference (VI), do not simply give a set of
samples, but an approximate \emph{distribution} $q\pp{\z\vert\x}$.
This turns out to enable diagnostics that would be impossible with
MCMC.

Our central idea is simple. Suppose that on input $\x$, inference
returns a distribution $q\pp{\z|\x}.$ Define the joint distribution
$q\pp{\z,\x}=p\pp{\x}q\pp{\z\vert\x}.$ Then, our diagnostic is an
estimate of $\SKL{p\pp{\zr,\xr}}{q\pp{\zr,\xr}},$ the symmetric KL-divergence
between $p\pp{\z,\x}$ and $q\pp{\z,\x}$. This essentially measures
how far $q\pp{\z\vert\x}$ is from $p\pp{\z\vert\x}$ on \emph{average},
over simulated datasets.

The key observation is that the symmetric divergence induces cancellations
induces cancellations between among the unknown normalization terms.
Specifically, if $\pp{\zr,\r x}\sim p\pp{\zr,\r x}$ and $\tilde{\r z}\sim q\pp{\zr\vert\xr}$
then we can simulate
\[
\r d=\log\frac{p\pp{\r z,\r x}}{q\pp{\r z|\r x}}-\log\frac{p\pp{\tilde{\r z},\r x}}{q\pp{\tilde{\r z}|\r x}},
\]
 and the expected value of $\r d$ is the symmetric divergence. To
compute this diagnostic, one must: (1) simulate $\pp{\z,\x}\sim p\pp{\z,\x}$
and $\tilde{\z}\sim q\pp{\z\vert\xr}$ and (2) compute $p\pp{\z,\x}$
and $q\pp{\z\vert\x}$. We do \emph{not} need to be able to evaluate
$p\pp{\x},$ even though it is part of the definitions of $p\pp{\z,\x}$
and $q\pp{\z,\x}$.

We also show that this idea can be extended to situations with conditional
or hidden variables. As an example of the latter, we show that it
can be used with importance-weighted inference methods that generate
many samples $\z\sim q(\z|\x)$ and select then one according to the
importance weights $p(\z,\x)/q(\z|\x)$ \citep{Burda_2015_ImportanceWeightedAutoencoders}.
Experiments show that the diagnostic gives practical measures of performance,
for regular VI, for Laplace's method, and for importance-weighted
variants of both.

\subsection{Notation}

The KL-divergence is $\KL{q\pp{\zr}}{p\pp{\zr}}=\E_{q\pp{\zr}}\log\pars{q\pp{\zr}/p\pp{\zr}}$
. Sans-serif font marks random variables. This disambiguates conflicting
conventions in machine learning and information theory. $\KL{q\pp{\zr\vert\x}}{p\pp{\zr\vert\x}}=\E_{q\pp{\zr\vert\x}}\log\pars{q\pp{\zr\vert\x}/p\pp{\zr\vert\x}}$
is a divergence over $\zr$ for a fixed $\x$. Meanwhile, $\KL{q\pp{\zr\vert\xr}}{p\pp{\zr\vert\xr}}=\E_{q\pp{\zr,\xr}}\log\pars{q\pp{\zr\vert\xr}/p\pp{\zr\vert\xr}}$
is the \emph{conditional divergence} \citep{Cover_2006_Elementsinformationtheory},
with an expectation over both $\zr$ and $\xr$. In all cases, symmetric
divergences are defined as $\SKL qp=\KL qp+\KL pq$.

\section{A New Simulation-Based Diagnostic}

This section gives a novel simulation-based diagnostic based on the
symmetric KL-divergence. The key idea is that some inference methods
(e.g. VI or Laplace's method) do not just give approximate samples,
but an approximate \emph{distribution} that can be evaluated at any
point. This enables certain diagnostics that would be impossible with
just a set of samples. Again, let $p\pp{\z,\x}$ be the target. We
consider approximate inference methods that input some $\x$ and produce
a distribution over $\z$ that approximates $p\pp{\z\vert\x}.$ We
denote that approximation as $q\pp{\z\vert\x}.$

One might hope to use the KL divergence $\KL{q\pp{\zr\vert\x}}{p\pp{\zr\vert\x}}$
as a diagnostic. This is almost never tractable since $p\pp{\x}$
is unknown. One can instead compute the evidence lower bound (ELBO)
$\E_{q\pp{\zr\vert\x}}\log(q\pp{\zr\vert\x}/p\pp{\zr,\x})$ which
is equal to the KL-divergence plus $\log p\pp x.$ The ELBO precisely
measures the \emph{relative} error for different algorithms, but gives
little information about the \emph{absolute} error, since $p\pp{\x}$
is unknown.

Instead, our diagnostic is based on the symmetric KL-divergence. The
basic idea of the diagnostic is to define a joint distribution $q\pp{\z,\x}=p\pp{\x}q\pp{\z\vert\x}$
with the same distribution over $\x$ as $p\pp{\z\vert\x}.$ Then,
cancellations make it possible to estimate the joint symmetric divergence
between $p\pp{\z,\x}$ and $q\pp{\z,\x}$. This is formalized in the
following result. 

\begin{restatable}{thm}{mainresult}

Given $p\pp{\z,\x}$ and $q\pp{\z\vert\x},$ define $q\pp{\z,\x}=p\pp{\x}q\pp{\z\vert\x}$.
Then,\label{thm:main-result}
\begin{align*}
\SKL{q\pp{\r z,\r x}}{p\pp{\r z,\r x}} & =\E\bracs{\log\frac{p\pp{\r z,\r x}}{q\pp{\r z|\r x}}-\log\frac{p\pp{\tilde{\r z},\r x}}{q\pp{\tilde{\r z}|\r x}}},
\end{align*}
where $\pp{\zr,\xr}\sim p\pp{\z,\x}$ is sampled from the model distribution
and $\tilde{\r z}\sim q\pp{\z|\x}$ is sampled from the approximating
distribution.

\end{restatable}

\noindent\parbox[t]{1\columnwidth}{%
\vspace{-1cm}%
\begin{minipage}[t]{0.49\textwidth}%
\noindent 
\begin{algorithm}[H]
For $k=1,2,\cdots,K$
\begin{itemize}
\item Simulate $\pp{\z,\x}\sim p\pp{\z,\x}$.
\item Infer $p\pp{\z,\x}$ to get $q(\z\vert\x)$. (fix $\x$)
\item Simulate $\tilde{\z}\sim q\pp{\z|\x}$. (fix $\x$)
\item ${\displaystyle d_{k}\leftarrow\log\frac{p\pp{\z,\x}}{q\pp{\z|\x}}-\log\frac{p\pp{\tilde{\z},\x}}{q\pp{\tilde{\z}|\x}}}$
\end{itemize}
Use $\mean(d_{1}\cdots d_{K})\approx\SKL{q\pp{\r z,\r x}}{p\pp{\r z,\r x}}.$

\vspace{0.35cm}

\caption{Computing the proposed diagnostic.\label{alg:simple-diagnostic}}
\end{algorithm}
\end{minipage}\hfill{}%
\begin{minipage}[t]{0.49\textwidth}%
\begin{algorithm}[H]
Input $\x$.

For $k=1,2,\cdots,K$
\begin{itemize}
\item Simulate $\pp{\z,\y}\sim p\pp{\z,\y\vert\x}$.
\item Infer $p\pp{\z,\y\vert\x}$ to get $q\pp{\z|\y,\x}$. (fix $\x,\y$)
\item Simulate $\tilde{\z}\sim q\pp{\z|\y,\x}$. ($\x,\y$ fixed)
\item ${\displaystyle d_{k}\leftarrow\log\frac{p\pp{\z,\y\vert\x}}{q\pp{\z|\y,\x}}-\log\frac{p\pp{\tilde{\z},\y\vert\x}}{q\pp{\tilde{\z}|\y,\x}}}$
\end{itemize}
Use $\mean(d_{1}\cdots d_{K})\approx\SKL{q\pp{\r z,\yr\vert\x}}{p\pp{\r z,\yr\vert\x}}.$

\caption{Diagnostic for conditional models.\label{alg:simple-diagnostic-conditional}}
\end{algorithm}
\end{minipage}%
}

Pseudocode for how one would use this result is given in \ref{alg:simple-diagnostic}.
One attractive aspect is that the output is the mean of a set of $K$
independent quantities. This makes it is easy to produce uncertainty
measures such as confidence intervals. These bound how far the estimated
diagnostic may be from the true symmetric divergence.

\subsection{Inference in Conditional Models}

Many inference problems are conditional, meaning one is given a model
$p\pp{\z,\y\vert\x}$, with no distribution specified over $\x$.
After observing $\x$ and $\y$, the goal is to predict $\z$. For
example, in regression or classification problems, $\x$ would represent
the input features, $\y$ the output values/labels, and $\z$ the
latent parameters.

In these cases, inference takes as input a pair $\pp{\y,\x}$ and
produces a distribution $q\pp{\z\vert\x,\y}$ approximating $p\pp{\z\vert\x,\y}$.
It's easy to see that the following generalization of \ref{thm:main-result}
holds. This is given by taking \ref{thm:main-result}, substituting
$\yr$ for $\xr$ and then conditioning all distributions on the fixed
value $\x$.
\begin{cor}
Given $p\pp{\z,\y\vert\x}$ and $q\pp{\z\vert\y,\x},$ define $q\pp{\z,\y\vert\x}=p\pp{\y\vert\x}q\pp{\z\vert\y,\x}$.
Then,\label{cor:main-result-conditional}
\begin{alignat*}{1}
\SKL{q\pp{\r z,\yr\vert\x}}{p\pp{\r z,\yr\vert\x}} & =\E\log\frac{p\pp{\r z,\yr\vert\x}}{q\pp{\r z|\yr,\x}}-\log\frac{p\pp{\tilde{\r z},\yr\vert\x}}{q\pp{\tilde{\r z}|\yr,\x}}
\end{alignat*}
where $\pp{\zr,\yr}\sim p\pp{\z,\y\vert\x}$ is sampled from the model
distribution and $\tilde{\r z}\sim q\pp{\z|\yr,\x}$ is sampled from
the approximating distribution.
\end{cor}

Pseudocode for how the diagnostic would be used with conditional models
is given as \ref{alg:simple-diagnostic-conditional}. It's critical
that $\x$ is \emph{not} a random variable-- it is the actual observed
input data. The simulated latent variables $\z$ and datasets $\y$
are conditioned on $\x$.

In order to run use this algorithm, one must be able to perform the
following operations: (1) Simulate $\pp{\zr,\yr}\sim p\pp{\z,\y\vert\x}$
and $\tilde{\zr}\sim q\pp{\z\vert\yr,\x}$. (2) Compute $p\pp{\z,\y\vert\x}$
for a given $\x$, $\y$, and $\z$. (3) Compute $q\pp{\z\vert\x}$
for a given $\x$, $\y$, and $\z$. It is not necessary to be able
to evaluate $p\pp{\y\vert\x},$ despite the fact that it is part of
the definition of $q\pp{\z,\x}$. 

\textbf{Example Results}. Before moving on to more complex cases,
we give some examples of the use of this diagnostic. \ref{fig:diagnostic-vi-laplace}
shows an example of running the diagnostic on five example models
using variational inference (VI) and Laplace's method that maximizes
$\log p\pp{\z,\x}$ to get $\hat{z}$ and uses a Gaussian centered
at $\hat{z}$ with a covariance matching the Hessian $H$ of $\log p$.
We also compare to an ``adjusted'' Laplace's method that better
matches the curvature if $\hat{z}$ is only an approximate maxima.
This instead uses a mean of $H^{-1}g$ where $g$ is that gradient
of $\log p$ at $\hat{z}.$ A more full description of the models
and inference algorithms is given in \ref{sec:Experiments}.

\begin{figure*}[t]
\begin{centering}
\includegraphics[scale=0.3]{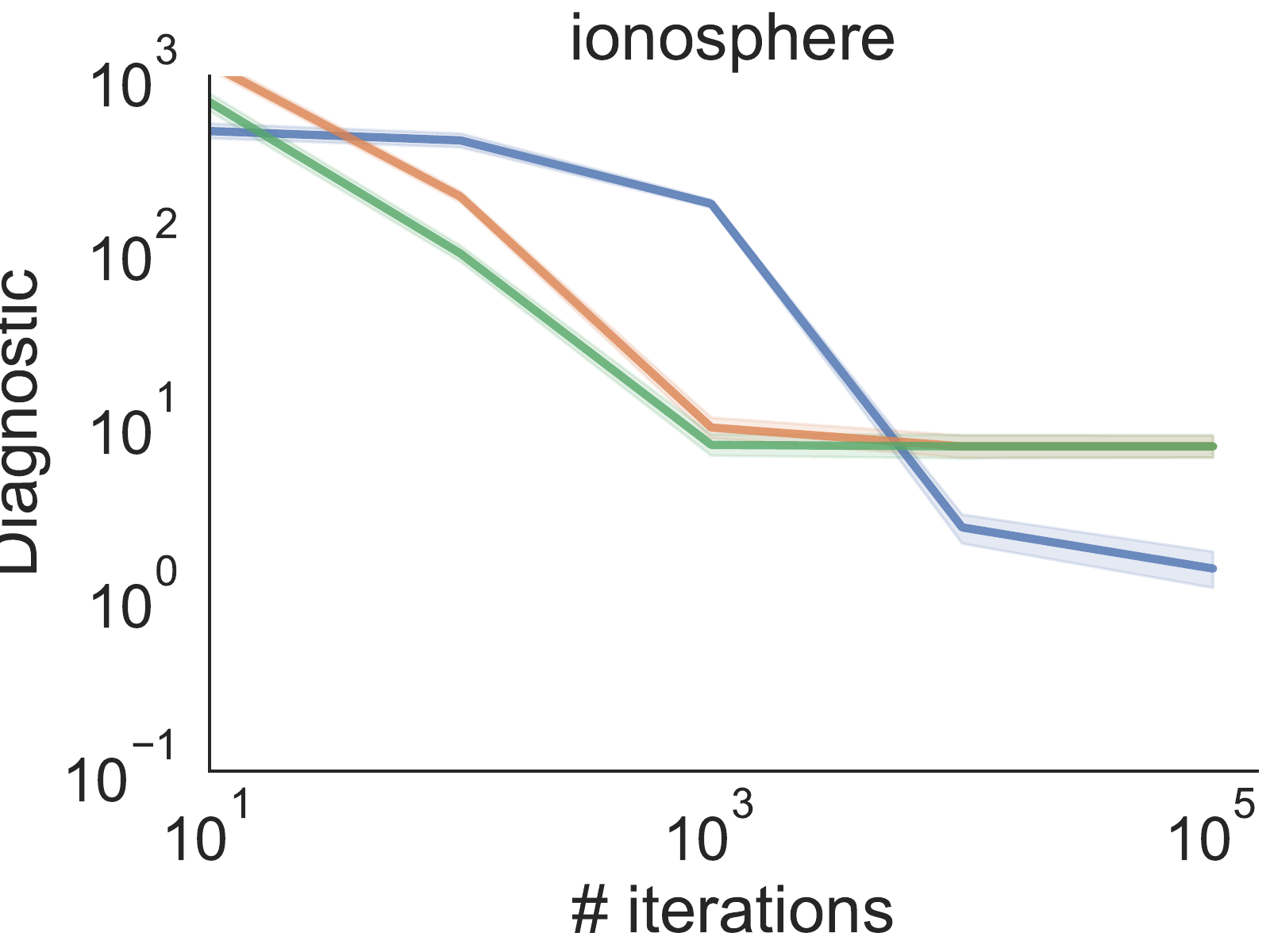}\includegraphics[viewport=32.27bp 0bp 461bp 346bp,clip,scale=0.3]{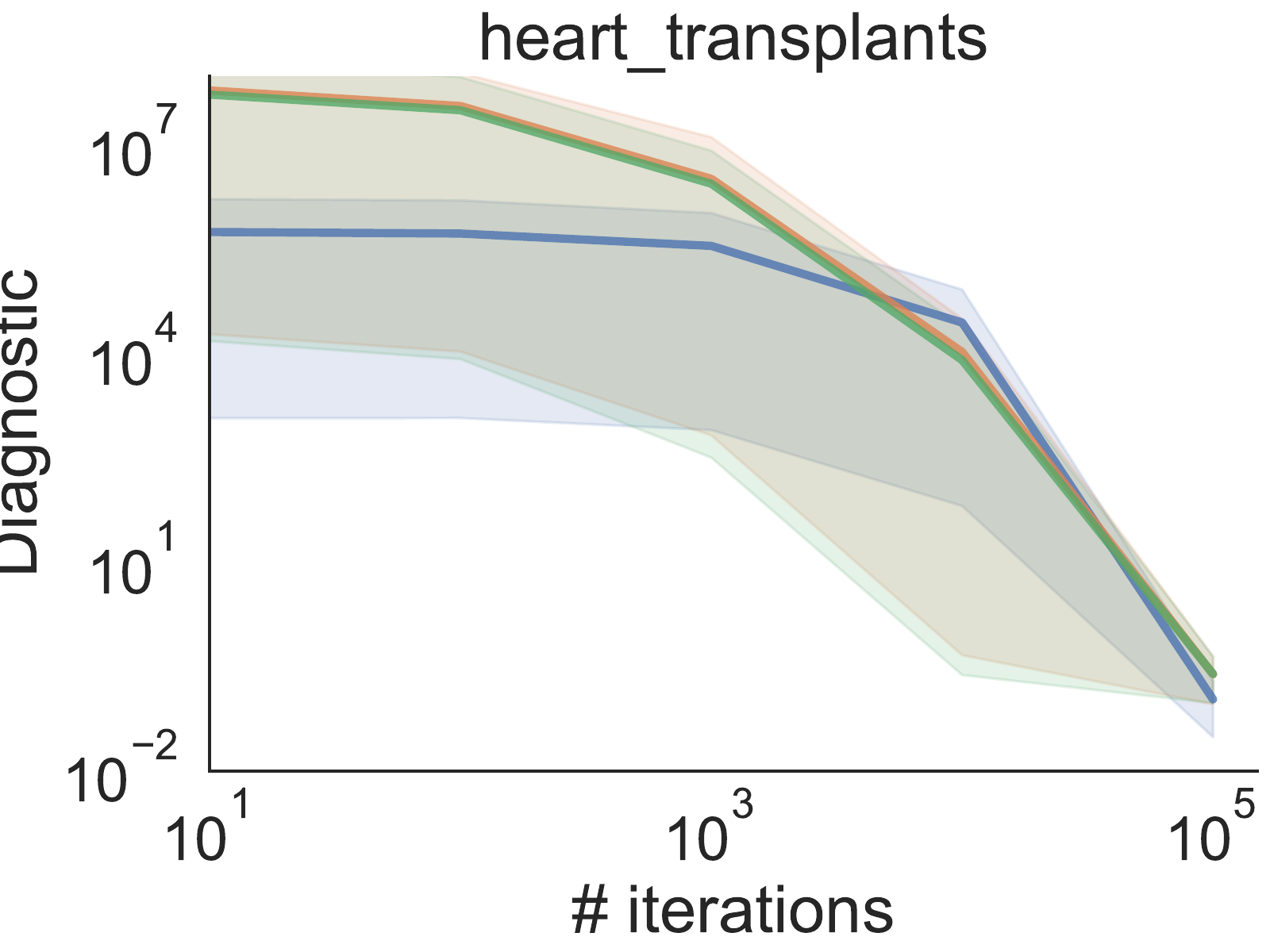}\includegraphics[viewport=32.27bp 0bp 461bp 346bp,clip,scale=0.3]{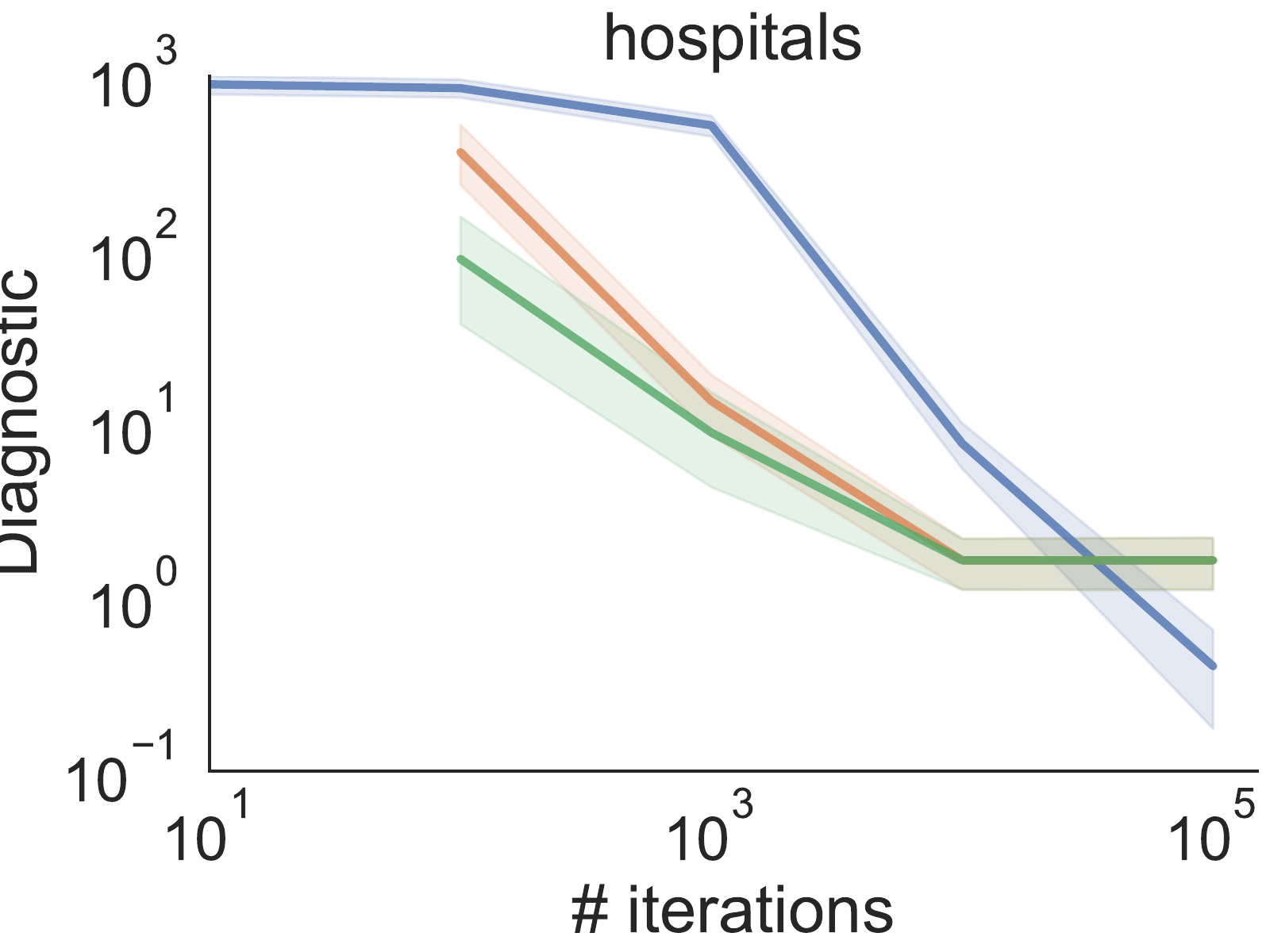}
\par\end{centering}
\begin{centering}
\includegraphics[scale=0.3]{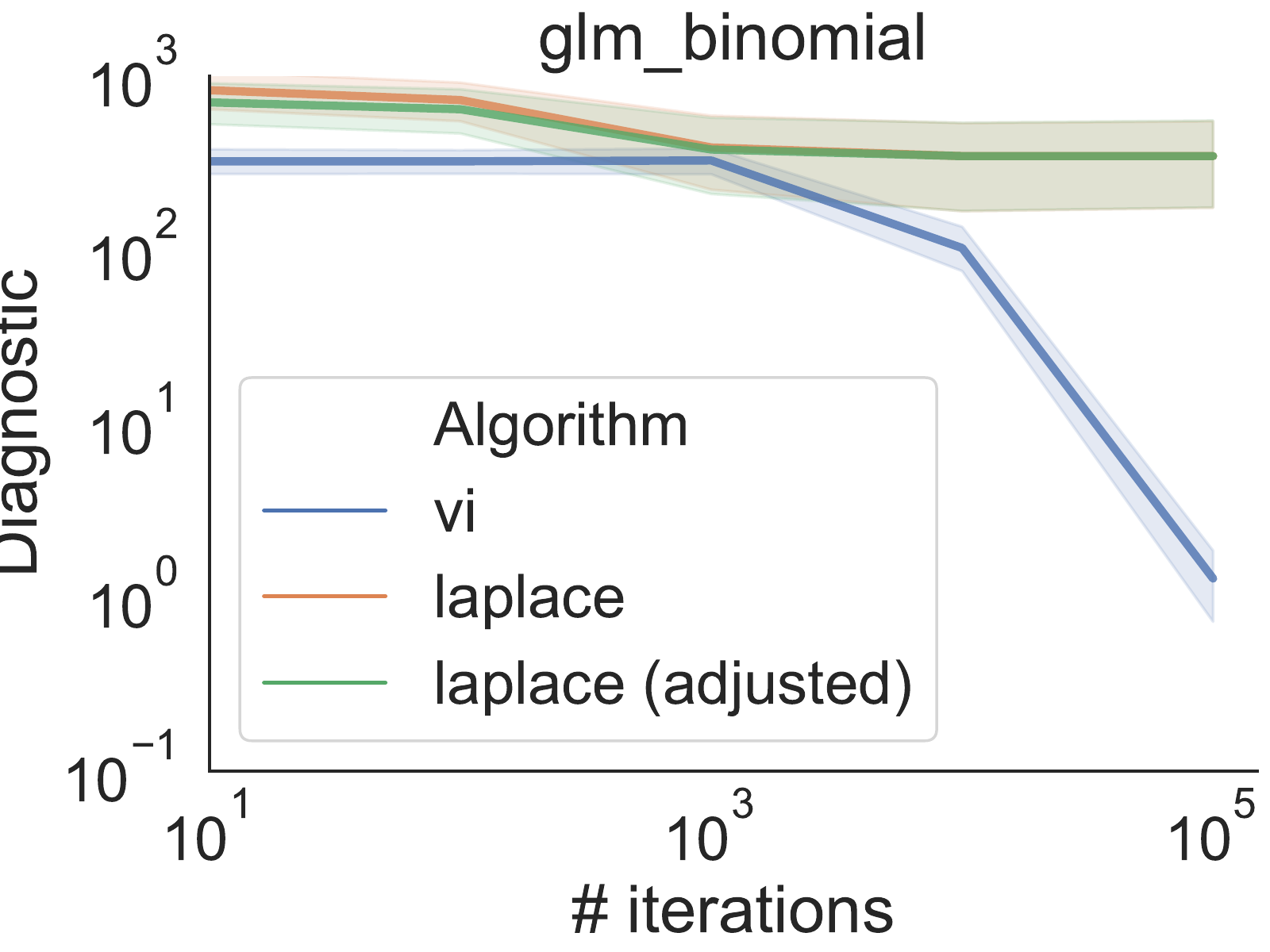}\includegraphics[viewport=27.66bp 0bp 461bp 346bp,clip,scale=0.3]{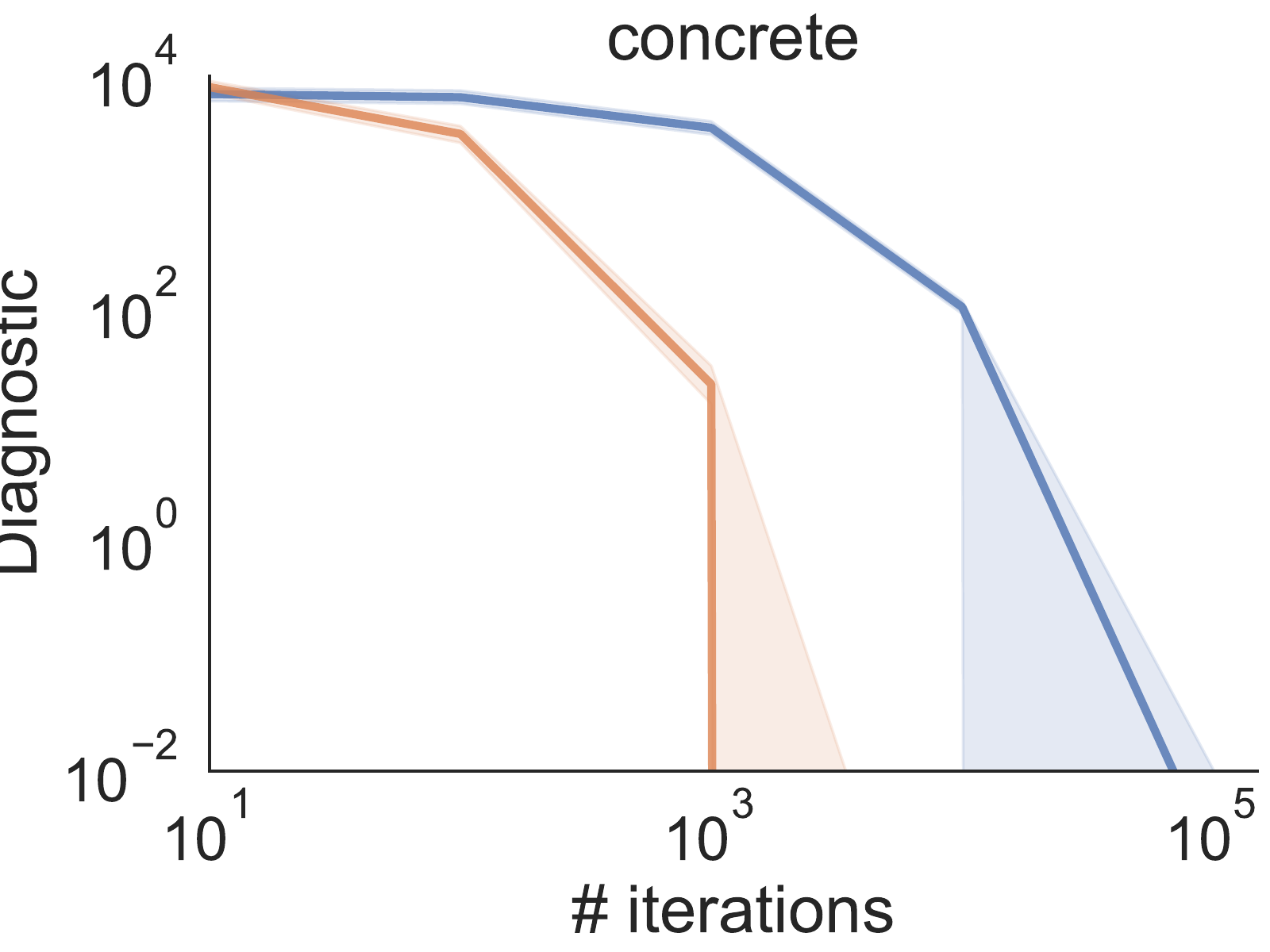}
\par\end{centering}
\caption{\textbf{The diagnostic gives plausible measures of the value of different
inference algorithms and of optimizing for different numbers of iterations\@.}
The diagnostic computed with $K=100$ repetitions. Lines show the
mean while the colored areas show 95\% confidence intervals. Confidence
intervals are computed before the log-transform and therefore appear
large for the lower-bounds seen on \texttt{concrete}. Laplace's method
fails due to numerical problems with few iterations on \texttt{hospitals}.
Adjusted Laplace's method is exact for \texttt{concrete}. \label{fig:diagnostic-vi-laplace}}
\end{figure*}

\section{Inference with Augmented Variables}

Many approximate inference methods used the idea of \emph{augmentation.}
The idea is to create an extra variable $\h$ and then approximate
$p\pp{\z,\h\vert\x}$ with $q\pp{\z,\h\vert\x}$. Why would this be
useful? The basic reason is that many powerful approximating families
are obtained by integrating out other random variables. Such families
often do not have tractable densities $q\pp{\z\vert\x}$, but can
be represented as the marginal of some density $q\pp{\z,\h\vert\x}.$
If we choose $p\pp{\h\vert\z,\x}$ in a way that is ``easy'' for
$q$ to match, then augmented inference might be nearly as accurate
as directly approximating $p\pp{\z\vert\x}$ with $q\pp{\z\vert\x}.$
\citet{Agakov_2004_AuxiliaryVariationalMethod} introduced the idea
of auxiliary variational inference, which fits this form. 

\noindent To apply the diagnostic to inference with hidden variables,
we need another version of \ref{thm:main-result}. This can be proven
by taking \ref{thm:main-result} and substituting $\pp{\z,\h}$ for
$\z$.
\begin{cor}
Given $p\pp{\z,\h,\x}$ and $q\pp{\z,\h\vert\x},$ define $q\pp{\z,\h,\x}=p\pp{\x}q\pp{\z,\h\vert\x}.$
Then\label{cor:main-result-hidden}
\begin{alignat*}{1}
\SKL{q\pp{\r z,\r h,\r x}}{p\pp{\r z,\r h,\r x}} & =\E\log\frac{p\pp{\r z,\r h,\r x}}{q\pp{\r z,\r h|\r x}}-\log\frac{p\pp{\tilde{\r z},\tilde{\r h},\r x}}{q\pp{\tilde{\r z},\tilde{\r h}|\r x}},
\end{alignat*}
where $\pp{\r z,\r x}\sim p\pp{\z,\x}$ is sampled from the model
distribution, $\r h\sim p\pp{\h|\zr,\xr}$ is sampled from the augmenting
distribution, and $\pp{\tilde{\r z},\tilde{\r h}}\sim q\pp{\z,\h|\xr}$
is sampled from the approximating distribution.
\end{cor}

The diagnostic is useful because (by the chain rule of KL-divergence),
\[
\SKL{q\pp{\r z,\r h,\r x}}{p\pp{\r z,\r h,\r x}}=\SKL{q\pp{\r z,\r x}}{p\pp{\r z,\r x}}+\SKL{q\pp{\r h\vert\r z,\r x}}{p\pp{\hr\vert\r z,\r x}},
\]
i.e. the diagnostic yields an upper bound on the error of $q\pp{\z,\x}.$ 

The corresponding algorithm is given as \ref{alg:diagnostic-hidden}.
This shows an important computational constraint. We assume that $p\pp{\z,\x}$
is given as a ``black box''. Thus, the augmented $p\pp{\z,\h,\x}$
must be chosen so that it is tractable to simulate $p\pp{\h\vert\z,\x}.$
The next section will give a special case of this result for a certain
class of approximate augmented inference methods.

\noindent\parbox[t]{1\columnwidth}{%
\vspace{-.5cm}

\begin{minipage}[t]{0.49\textwidth}%
\noindent 
\begin{algorithm}[H]
For $k=1,2,\cdots,K$:
\begin{itemize}
\item Simulate $\pp{\z,\x}\sim p\pp{\z,\x}$.
\item Infer $p\pp{\z,\x}p\pp{\b h|\z,\x}$ to get $q\pp{\z,\b h|\x}$. (fix
$\x$) 
\item Simulate $\b h\sim p\pp{\b h|\z,\x}.$
\item Simulate $\pp{\tilde{\z},\tilde{\b h}}\sim q\pp{\z,\b h|\x}$. (fix
$\x$)
\item ${\displaystyle d_{k}\leftarrow\log\frac{p\pp{\z,\b h,\x}}{q\pp{\z,\b h|\x}}-\log\frac{p\pp{\tilde{\z},\tilde{\b h},\x}}{q\pp{\tilde{\z},\tilde{\b h}|\x}}}$
\end{itemize}
Use $\mean(d_{1}\cdots d_{K})\approx\SKL{q\pp{\r z,\r h,\r x}}{p\pp{\r z,\r h,\r x}}.$\vspace{0.48cm}

\caption{Diagnostic with augmentation.\label{alg:diagnostic-hidden}}
\end{algorithm}
\end{minipage}\hfill{}%
\begin{minipage}[t]{0.49\textwidth}%
\noindent 
\begin{algorithm}[H]
For $k=1,2,\cdots,K$:
\begin{itemize}
\item Simulate $\pp{\z_{1},\x}\sim p\pp{\z,\x}$.
\item Run inference to find a base distribution $q\pp{\z\vert\x}$.
\item Simulate $\z_{2},\cdots,\z_{M}\sim q\pp{\z\vert\x}.$
\item Simulate $\tilde{\z}_{1},\cdots,\tilde{\z}_{M}\sim q\pp{\z|\x}$.
(fix $\x$)
\item ${\displaystyle d_{k}\leftarrow\log\sum_{m=1}^{M}\frac{p\pp{\z_{m},\x}}{q\pp{\z_{m}\vert\x}}}$${\displaystyle -\log\sum_{m=1}^{M}\frac{p\pp{\tilde{\z}_{m},\x}}{q\pp{\tilde{\z}_{m}\vert\x}}}$
\end{itemize}
Use $\mean(d_{1}\cdots d_{K})\approx$

\hfill{}$\SKL{q_{IW}\pp{\r z_{1}\cdots\zr_{M},\r x}}{p_{IW}\pp{\r z_{1}\cdots\zr_{M},\r x}}.$

\caption{With importance-weighting.\label{alg:diagnostic-IW}}
\end{algorithm}
\end{minipage}%
}

\section{Importance Sampling}

Self-normalized importance sampling is a classic Monte-Carlo method
\citep{Owen_2013_MonteCarlotheory}. Given any distribution $q\pp{\z\vert\x}$,
one can approximately sample from the posterior $p\pp{\z\vert\x}$
by drawing a set of $M$ samples $\hat{\z}_{1},\cdots,\hat{\z}_{M}\sim q\pp{\z\vert\x}$,
selecting one with probability proportional to the importance weights
$p\pp{\hat{\z}_{m},\x}/q\pp{\hat{\z}_{m}}$ and then returning the
final sample $\z_{1}=\hat{\z}_{m}.$ Call the resulting density $q_{IW}\pp{\z_{1}\vert\x}.$

One might hope to directly apply the diagnostic to self-normalizing
importance sampling. However, this cannot be done because it is intractable
to to evaluate $q_{IW}\pp{\z_{1}\vert\x}.$ However, we can identify
\emph{augmented} distributions, and thereby upper-bound the symmetric
divergence. Define the distribution

\begin{equation}
p_{IW}\pp{\z_{1},\cdots,\z_{M},\x}=p\pp{\z_{1},\x}\prod_{m=2}^{M}q\pp{\z_{m}\vert\x}.\label{eq:p_IW}
\end{equation}
This can be seen as an augmented distribution with $\z_{2},\cdots,\z_{M}$
the hidden variables augmenting the original $\z_{1}.$ Define also
\begin{equation}
q_{IW}\pp{\z_{1},\cdots,\z_{M}\vert\x}=\frac{p_{IW}\pp{\z_{1},\cdots,\z_{M},\x}}{\frac{1}{M}\sum_{m=1}^{M}\frac{p\pp{\z_{m},\x}}{q\pp{\z_{m}\vert\x}}}.\label{eq:q_IW}
\end{equation}

It is not immediately obvious that this augments the self-normalized
importance sampling density $q_{IW}\pp{\z_{1}\vert\x}$ introduced
at the beginning of this section (or indeed that this is a valid density
at all). However, it was recently shown \citep{Domke_2018_ImportanceWeightingVariational}
that the following algorithm samples from $q_{IW}.$
\begin{claim}
The following process yields a sample from $q_{IW}\pp{\z_{1},\cdots,\z_{M}\vert\x}$
as defined in \ref{eq:q_IW}.
\end{claim}

\begin{enumerate}
\item Draw $\hat{\zr}_{1},\cdots\hat{\zr}_{M}\sim q\pp{\z\vert\x}$.
\item Choose $m\in\left\{ 1,\cdots,M\right\} $ with $\text{\ensuremath{\P}}\bb m\propto\frac{p\pp{\hat{\zr}_{m},\x}}{q\pp{\hat{\zr}_{m}}}$
\item Set $\left(\zr_{1},\cdots,\zr_{M}\right)=\left(\hat{\zr}_{m},\hat{\zr}_{1},\cdots,\hat{\zr}_{m-1},\hat{\zr}_{m+1},\cdots,\hat{\zr}_{M}\right)$
\end{enumerate}
Informally, this algorithm draws $m$ samples from $q(\z\vert\x)$
then swaps one to be in the first position, chosen according to importance
weights. Thus, this is a valid augmentation of the self-normalized
importance sampling distribution.

This shows that \ref{eq:p_IW} and \ref{eq:q_IW} augment the target
and self-normalized importance-sampling density, and thus that $\SKL{q_{IW}\pp{\zr_{1},\cdots,\zr_{M},\xr}}{p_{IW}\pp{\zr_{1},\cdots,\zr_{M},\xr}}$
upper-bounds $\SKL{q_{IW}\pp{\zr_{1},\xr}}{p\pp{\zr_{1},\xr}}.$ It
can be shown that the (non-symmetric) divergence from $q_{IW}$ to
$p_{IW}$ asymptotically decreases at a $1/M$ rate \citep{Maddison_2017_FilteringVariationalObjectives,Domke_2018_ImportanceWeightingVariational}.

\begin{figure*}[t]
\begin{centering}
\includegraphics[viewport=0bp 0bp 461bp 346bp,clip,scale=0.3]{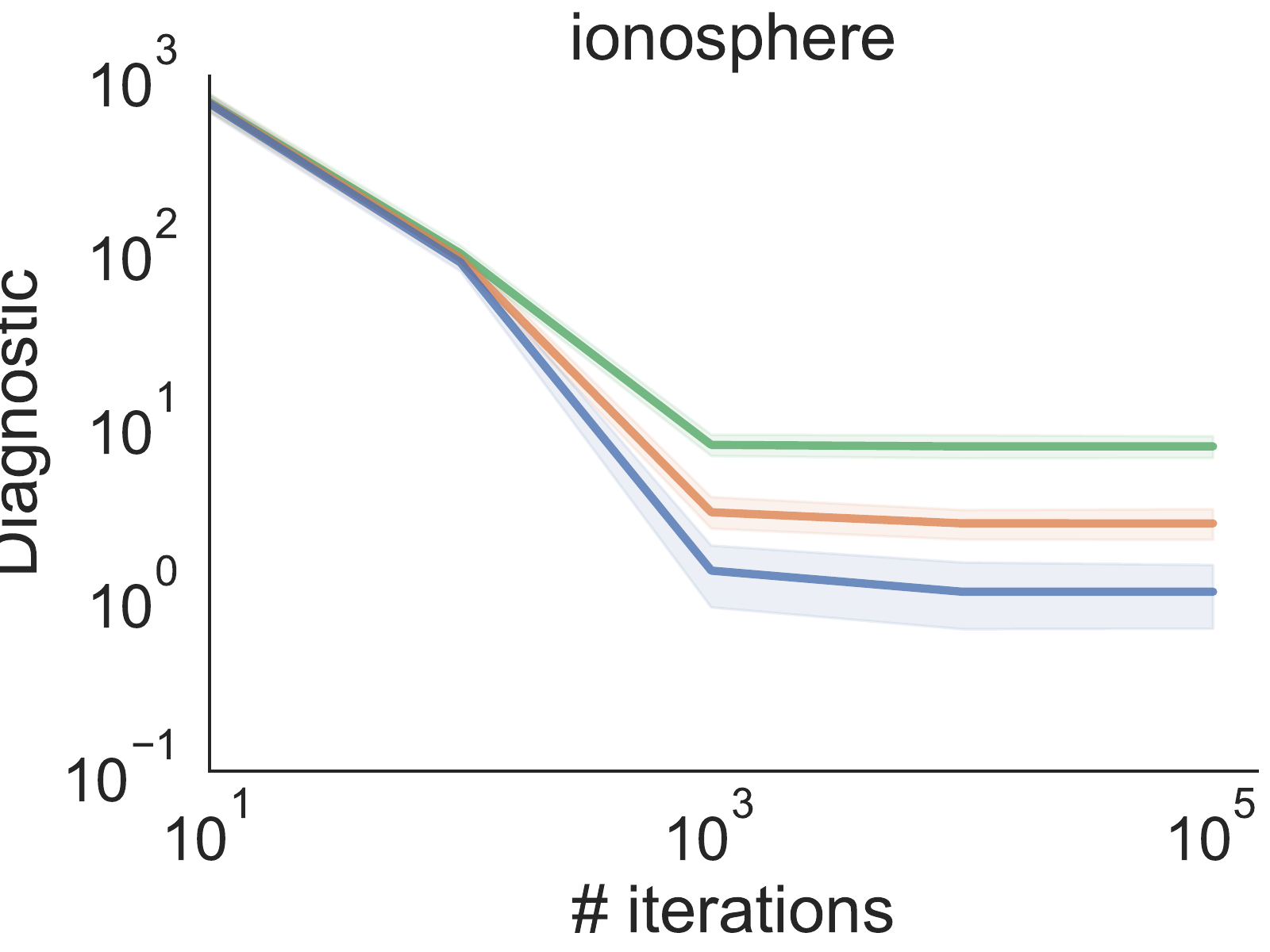}\includegraphics[viewport=32.27bp 0bp 461bp 346bp,clip,scale=0.3]{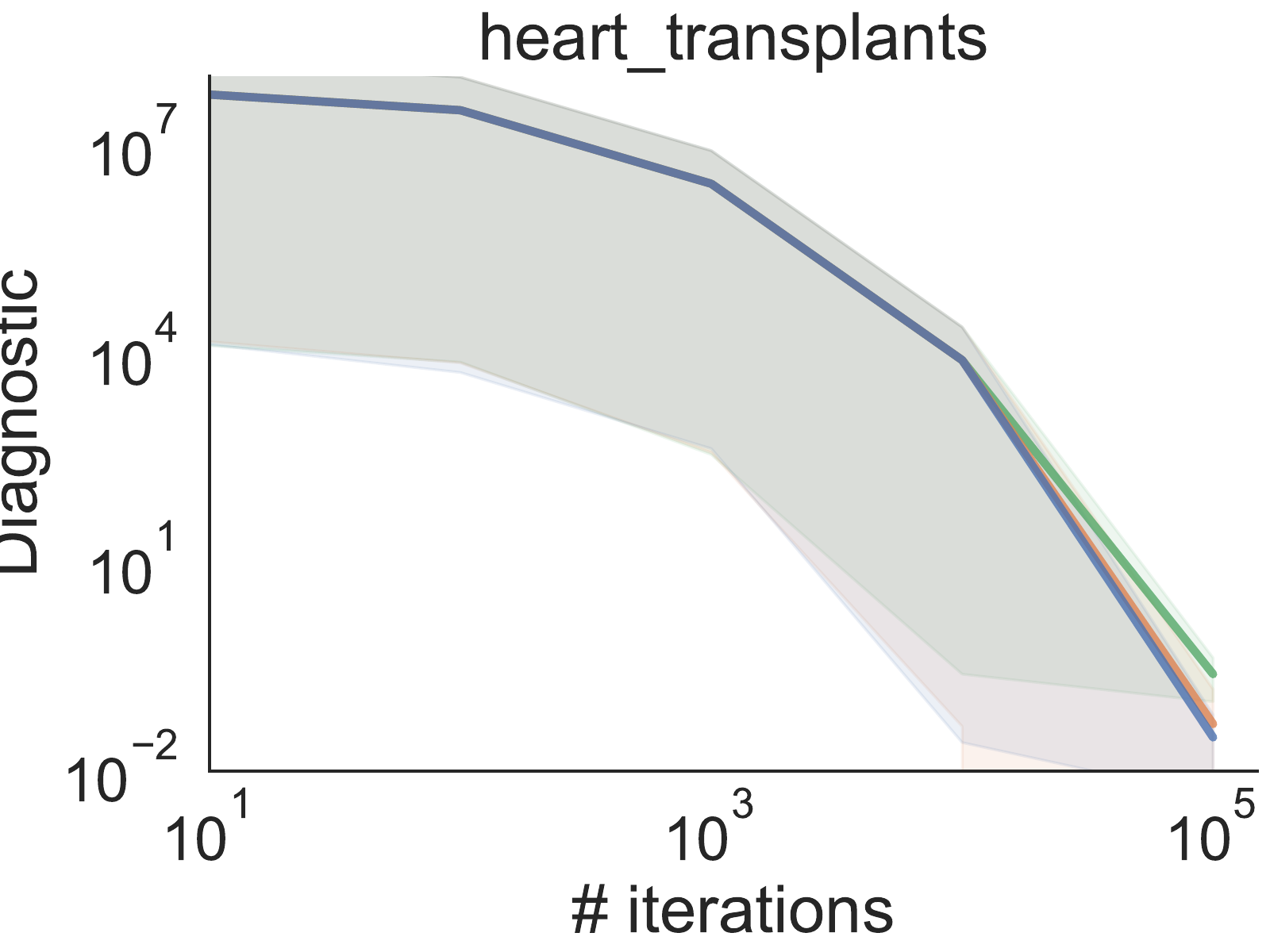}\includegraphics[viewport=32.27bp 0bp 461bp 346bp,clip,scale=0.3]{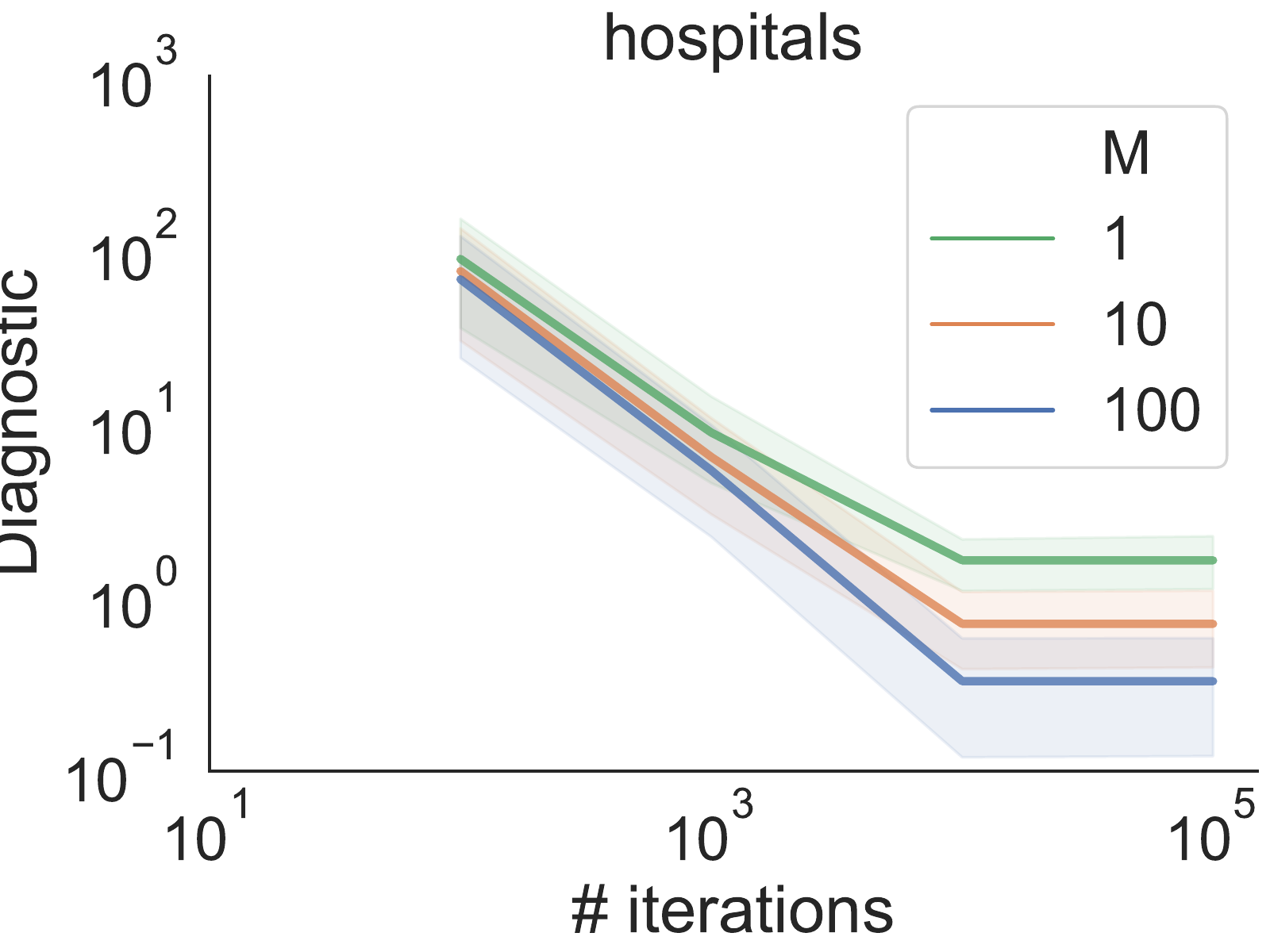}
\par\end{centering}
\caption{The diagnostic applied to self-normalized importance-weighting with
$M$ samples. The proposal distribution is estimated by Laplace's
method with adjustment. The $M=1$ curve is the same as the \textquotedblleft laplace
(adjusted)\textquotedblright{} curve in \ref{fig:diagnostic-vi-laplace}.
\label{fig:diagnostic-importance-weighting-laplace}}
\end{figure*}

\subsection{Diagnosing Importance Weighted Inference}

The above discussion shows that adding importance-sampling can improve
any distribution $q.$ Even better results can be obtained by explicitly
optimizing $q$ to work well after augmentation. Importance-weighted
variational inference directly performs an optimization to maximize
the ELBO from $q_{IW}$ to $p_{IW},$ equivalent to minimizing the
KL-divergence $\KL{q_{IW}}{p_{IW}}$ between the augmented distributions.
This ELBO can be simplified using the relationship that

\begin{multline}
\E_{q_{IW}\pp{\zr_{1},\cdots,\zr_{M}\vert\x}}\log\frac{p_{IW}\pp{\zr_{1},\cdots,\zr_{M},\x}}{q_{IW}\pp{\zr_{1},\cdots,\zr_{M}\vert\x}}=\E_{q\pp{\zr_{1}\vert\x}}\cdots\E_{q\pp{\zr_{M}\vert\x}}\log\frac{1}{M}\sum_{m=1}^{M}\log\frac{p\pp{\zr_{m},\x}}{q\pp{\zr_{m}\vert\x}},\label{eq:IW-ELBO}
\end{multline}
which follows from cancellations between $p_{IW}$ and $q_{IW}$,
followed from the observation that the argument of the expectation
is constant with respect to permutations of $\zr_{1},\cdots,\zr_{m}.$
This objective was originally introduced in the context of importance-weighted
auto-encoders \citep{Burda_2015_ImportanceWeightedAutoencoders,Cremer_2017_ReinterpretingImportanceWeightedAutoencoders,Domke_2018_ImportanceWeightingVariational,Naesseth_2018_VariationalSequentialMonte,Le_2018_AutoEncodingSequentialMonte}
(without explicitly identifying $q_{IW}$ and $p_{IW}$) and subsequently
studied by various others \citep{Cremer_2017_ReinterpretingImportanceWeightedAutoencoders,Domke_2018_ImportanceWeightingVariational,Bachman_2015_TrainingDeepGenerative,Naesseth_2018_VariationalSequentialMonte,Le_2018_AutoEncodingSequentialMonte}.

The following result specializes \ref{cor:main-result-hidden} to
the case of importance-weighted inference.

\begin{restatable}{cor}{mainIWVI}

Given $p\pp{\z,\x}$ and $q\pp{\z\vert\x},$ define $p_{IW}$ and
$q_{IW}$ as in \ref{eq:p_IW} and \ref{eq:q_IW}. Further, set $q_{IW}\pp{\z_{1},\cdots,\z_{M},\x}=p\pp{\x}q_{IW}\pp{\z_{1},\cdots,\z_{M}\vert\x}.$
Then\label{cor:main-result-IWVI}\vspace{-0.5cm}

\begin{alignat*}{1}
\SKL{q_{IW}\pp{\r z_{1},\cdots,\zr_{M},\r x}}{p_{IW}\pp{\r z_{1},\cdots,\r x}} & =\E\log\frac{1}{M}\sum_{m=1}^{M}\frac{p\pp{\z_{m},\x}}{q\pp{\z_{m}\vert\x}}-\log\frac{1}{M}\sum_{m=1}^{M}\frac{p\pp{\tilde{\z}_{m},\x}}{q\pp{\tilde{\z}_{m}\vert\x}},
\end{alignat*}
where $\pp{\r z_{1},\r x}\sim p\pp{\z,\x}$ is sampled from the model
distribution, $\zr_{2},\cdots,\zr_{M}\sim q\pp{\z\vert\xr}$ is sampled
from the approximating distribution and $\tilde{\zr}_{1},\cdots,\tilde{\zr}_{M}\sim q\pp{\z\vert\xr}$
are also sampled from the approximating distribution.

\end{restatable}

A proof is in the supplement. Unlike \ref{cor:main-result-conditional}
and \ref{cor:main-result-hidden}, the result is not trivial. The
main idea is to substitute $p_{IW}\pp{\z_{1},\cdots,\z_{M},\x}$ for
$p\pp{\z,\h,\x}$ and $q_{IW}\pp{\z_{1},\cdots,\z_{M}\vert\x}$ for
$q\pp{\z,\h\vert\x}.$ Then, many expressions can be simplified based
on the particular forms of $p_{IW}$ and $q_{IW}$. Finally, we can
observe that the argument of the expectation is independent of permutations
of $\tilde{\z}_{1},\cdots,\tilde{\z}_{M}$. This allows a final simplification.

\section{Experiments\label{sec:Experiments}}

\begin{figure*}[t]
\begin{centering}
\includegraphics[scale=0.3]{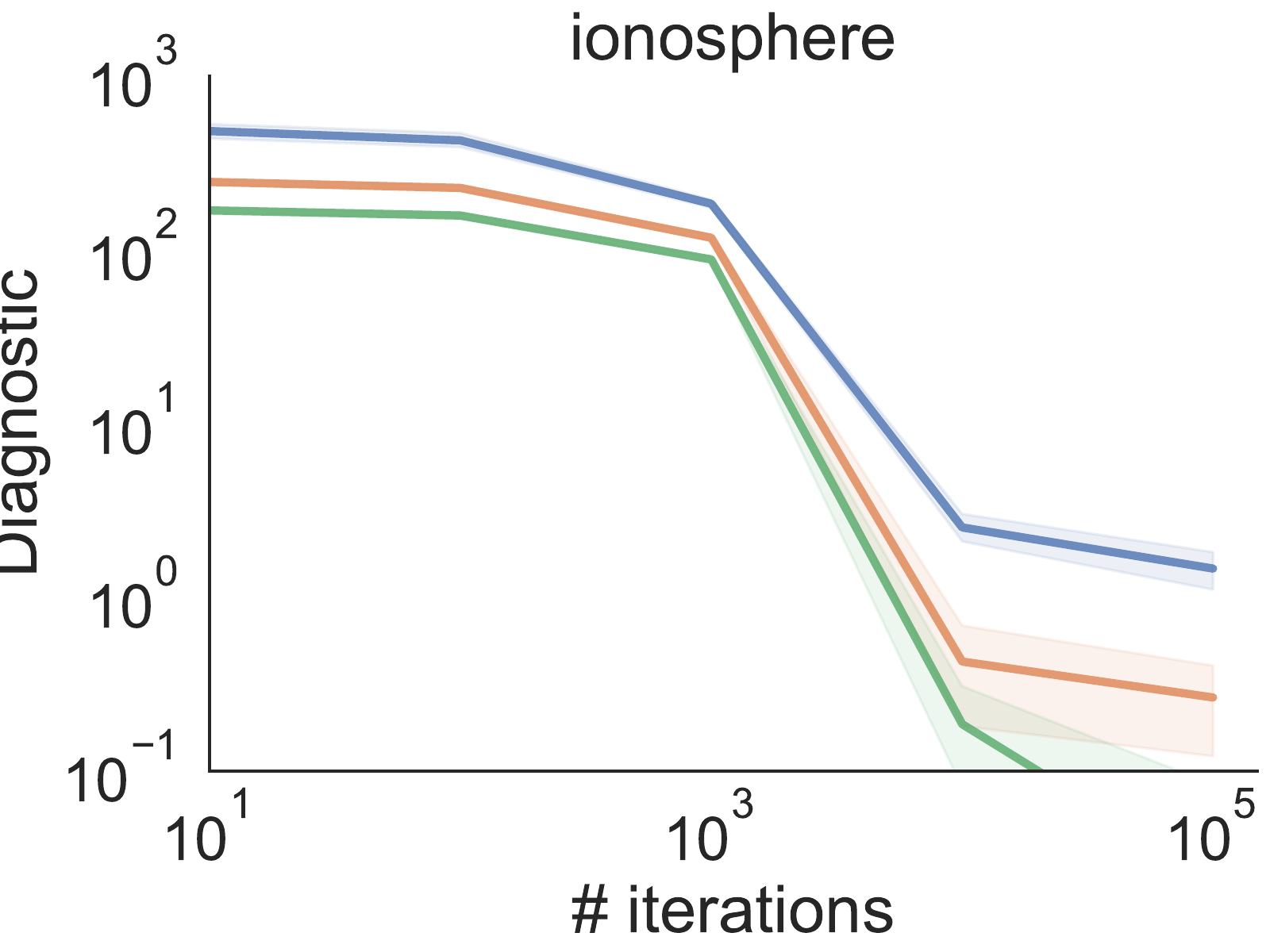}\includegraphics[viewport=32.27bp 0bp 461bp 346bp,clip,scale=0.3]{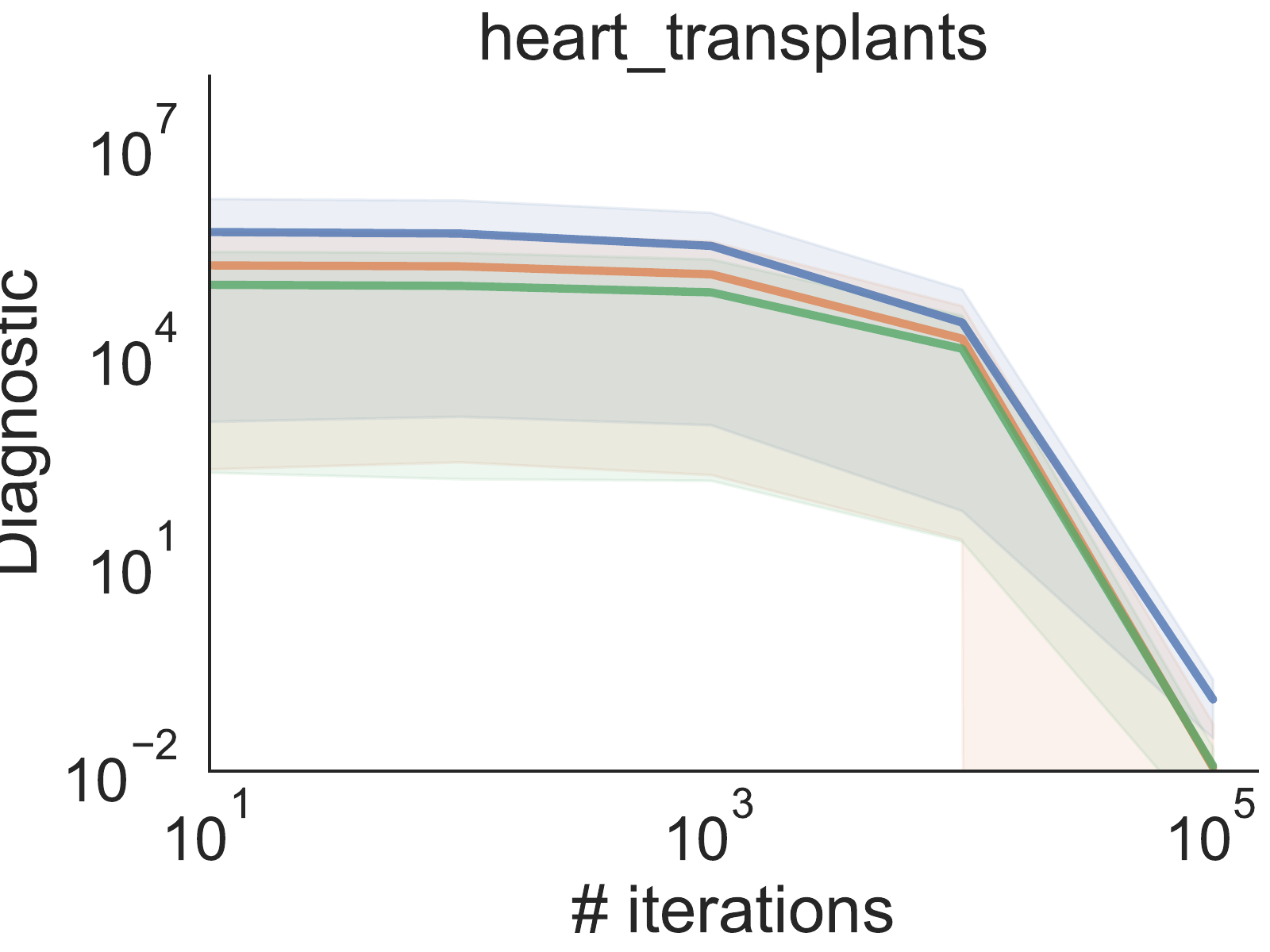}\includegraphics[viewport=32.27bp 0bp 461bp 346bp,clip,scale=0.3]{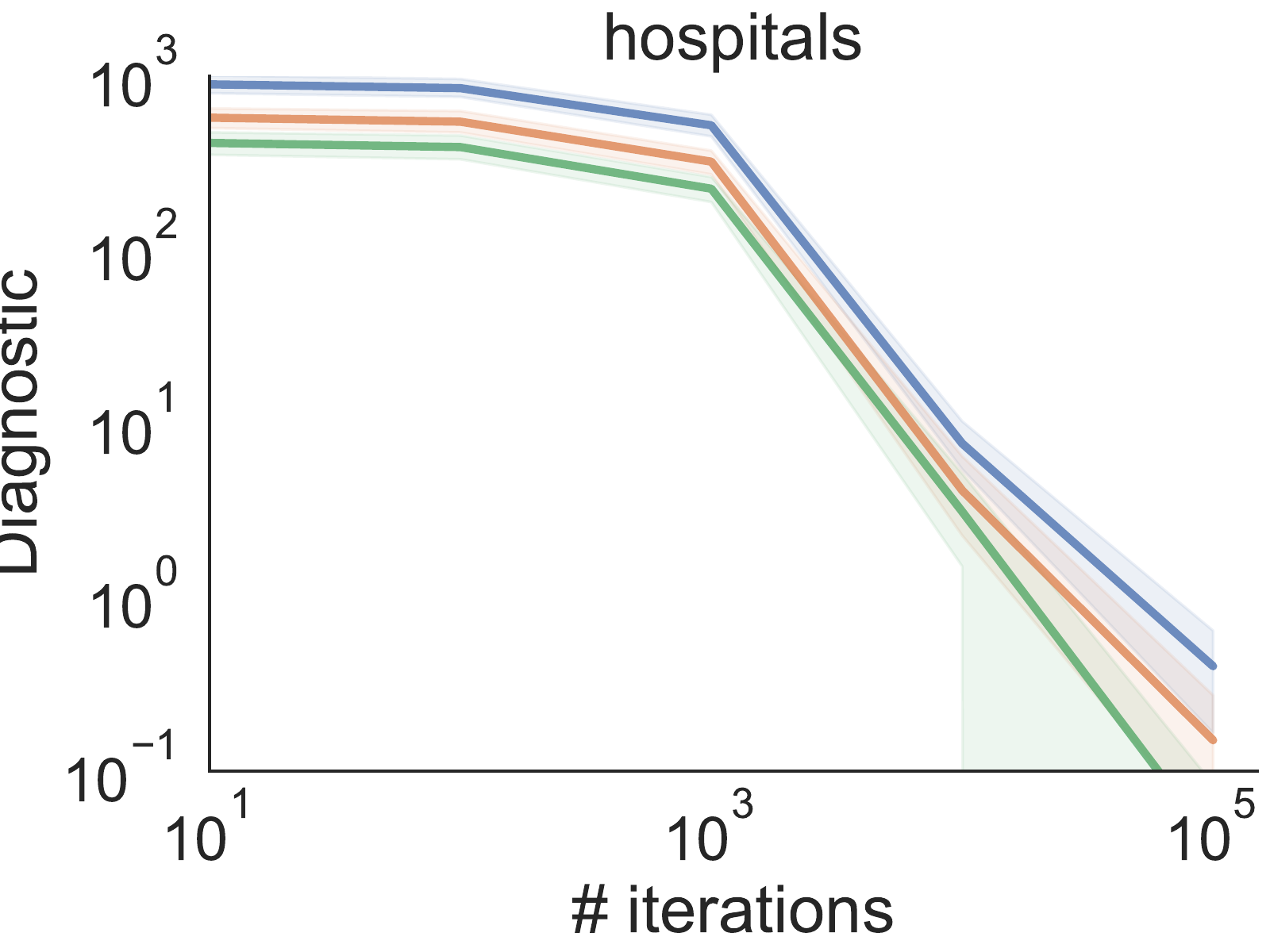}
\par\end{centering}
\begin{centering}
\includegraphics[scale=0.3]{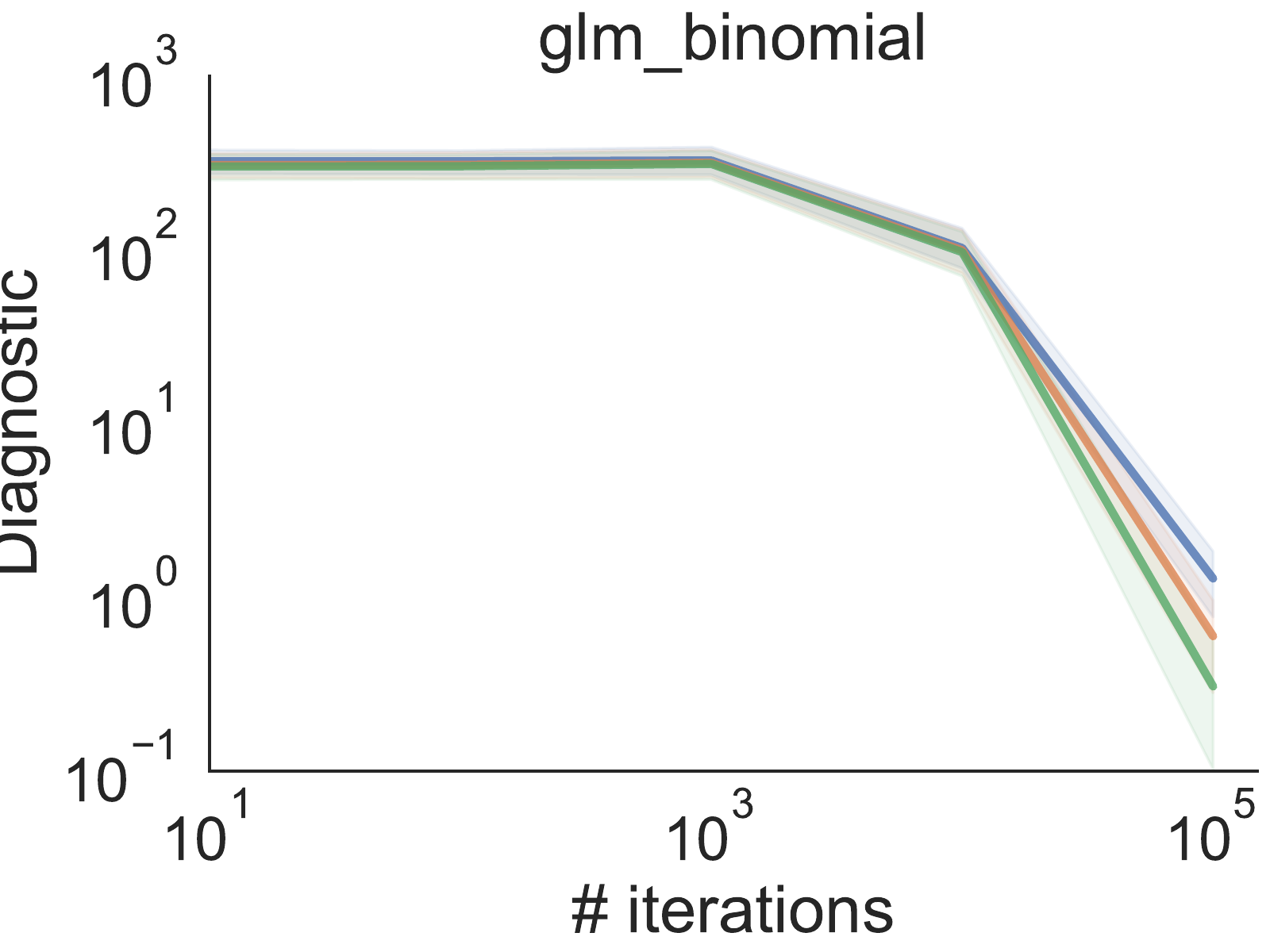}\includegraphics[viewport=23.05bp 0bp 461bp 346bp,clip,scale=0.3]{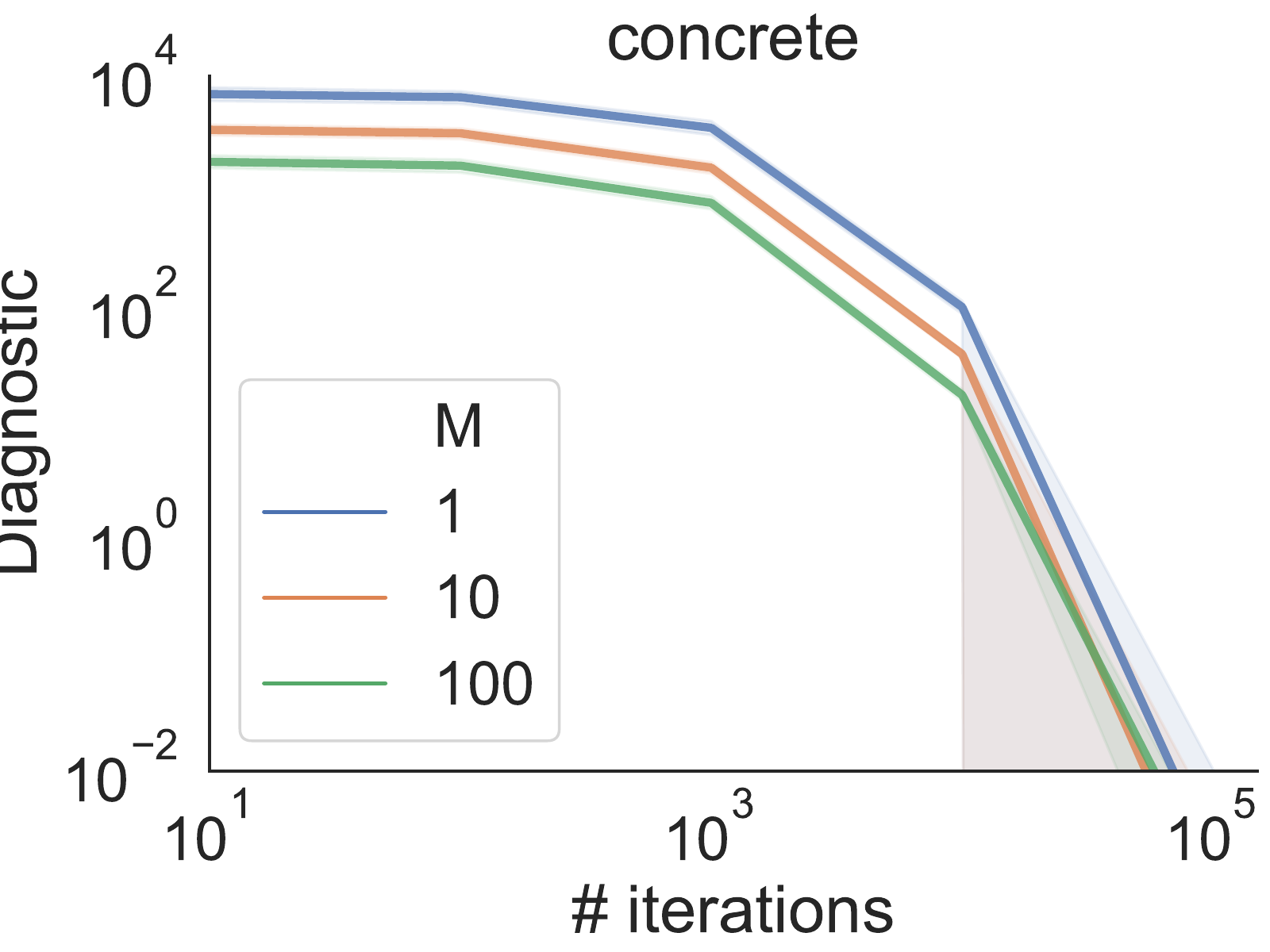}
\par\end{centering}
\caption{The diagnostic as computed using \ref{alg:diagnostic-IW} where parameters
are determined by optimizing \ref{eq:IW-ELBO} with the same value
$M$. The $M=1$ curve is the same as the \textquotedblleft vi\textquotedblright{}
curve in \ref{fig:diagnostic-vi-laplace}. \label{fig:diagnostic-importance-weighting}}
\end{figure*}

\textbf{Models.} We use five models, described in detail in Sec. \ref{sec:Models}
(Supplement). \texttt{glm\_binomial} is a hierarchical model of the
number of a bird population over time. \texttt{heart\_transplants}
models the survival times of patients after surgery. \texttt{hospitals}
measures the number of deaths in different hospitals. \texttt{ionosphere}
is a Bayesian logistic regression model on a classic dataset. \texttt{concrete}
is a Bayesian linear regression model -- included as a baseline because
the exact posterior is Gaussian.

\textbf{Optimization}. The first inference algorithm we consider is
Laplace's method which produces a multivariate Gaussian approximation
$q\pp{\z\vert\x}$. For this method, we run Adam with a step-size
of 0.01 for the first half of iterations, and 0.001 for the second
half. The resulting point $\hat{\z}$ is the mean. Finally, the Hessian
of $\log p$ at $\hat{\z}$ is used to estimate the covariance of
$q$ as $\Sigma=-H^{-1}.$ This method fails to match the local curvature
of $\log p$ when $\hat{\z}$ is far from an optima. We also consider
an ``adjusted'' Laplace's method that instead uses a mean of $H^{-1}g$
where $g$ is that gradient of $\log p$ at $\hat{\z}$. This guarantees
that $\log q$ has the same gradient at $\hat{\z}$ as $\log p$. 

We also consider variational inference. We initialize to a standard
Gaussian and optimize with Adam with a step size of 0.001 for the
first half of iterations, and 0.0001 for the second half. We estimate
the gradient of the ELBO using the reparameterization trick, using
the ``sticking the landing'' estimator \citep{Roeder_2017_StickingLandingSimple}.

\textbf{Constraints and Transformations}. Often, random variables
have constraints -- they are not supported over the reals. As is
common \citep{Kucukelbir_2017_AutomaticDifferentiationVariational}
we deal with this through a process of \emph{transformations}. For
our models, it is sufficient to consider two cases:
\begin{itemize}
\item Random variables $\r x$ that are either defined over the non-negative
reals $[0,\infty)$. In this case, we replace $\r x$ with a new random
variable $\r x'=\log\r x$ that is unconstrained.
\item Random variables $\r x$ defined on a closed interval $\bb{a,b}.$
Here, we transform to $\r x'=\mathrm{Logit}\pp{\frac{\r x-a}{a-b}}.$
\end{itemize}
\textbf{Results}. Results comparing VI and the two variants of Laplace's
method are shown in \ref{fig:diagnostic-vi-laplace}, averaging over
$K=100$ simulated datasets. Laplace's method is reasonably accurate
in many cases, but usually has a ``floor'' of accuracy it does not
exceed. The adjustment to Laplace's method is often helpful and never
harmful. VI performs better with many iterations. For these models,
the diagnostic shows that inference error is reasonably low with many
iterations, but not quite ``exact''.

\ref{fig:diagnostic-importance-weighting-laplace} shows the results
of importance sampling with a proposal computed using Laplace's method
with adjustment. Using more samples yields a clear improvement. 

Finally, results with full importance weighted variational inference
(optimizing the \ref{eq:IW-ELBO} rather than the standard ELBO) is
shown in \ref{fig:diagnostic-importance-weighting}. The same value
$M$ is used during optimization and at test time.

\section{Discussion}

This paper proposed a new diagnostic for approximate inference methods.
This is a simulation-based diagnostic, meaning it is computed by repeatedly
simulating latent variables along with datasets and running inference
on each dataset. The central idea is that cancellations in unknown
constants make it possible to estimate a symmetric divergence. This
is notable in being a simple, scalar quantity with a clear information-theoretic
interpretation. It can also be computed in a fully automated way along
with error measures like confidence intervals. We showed that the
diagnostic can be extended to augmented inference methods, in particular
importance-weighted inference. Empirically, the method gives reasonable
diagnostic information on several test models.

While \ref{thm:main-result} is quite simple, there are numerous points
worth clarifying in its use as a diagnostic:

\textbf{What the diagnostic measures}. One possibly counter-intuitive
aspect of this diagnostic (like all simulation-based diagnostics)
is that it does not use the actual observed data $\x$. Rather, it
measures the \emph{typical} error, averaged over $\x$ simulated from
the model $p\pp{\x}$. It is this essential that the prior $p\pp{\z}$
and likelihood $p\pp{\x\vert\z}$ be selected so that $p\pp{\x}=\int p\pp{\z}p\pp{\x\vert\z}d\z$
yields realistic simulated datasets. In particular, very broad priors
are might lead to ``nonsense'' observations $\x$ that are unrepresentative
of the data that would be seen in practice.

\textbf{Computational considerations}. In order to compute this diagnostic,
one must be able to perform several operations: (1) Simulate $\pp{\zr,\xr}\sim p\pp{\z,\x}$
and $\tilde{\zr}\sim q\pp{\z\vert\xr}$. (2) Compute $p\pp{\z,\x}$
for a given $\z$ and $\x$. (3) Compute $q\pp{\z\vert\x}$ for a
given $\z$ and $\x$. Crucially, it is \emph{not} necessary to be
able to evaluate $p\pp{\x}$. This is true despite the fact that $p\pp{\x}$
is part of the definition of $q\pp{\z,\x}$.

\textbf{Other representations of the diagnostic.} The quantity that
the diagnostic is representing can be written in a different form
that emphasizes that it measures errors over $\z$. This uses the
notion of a conditional divergence $\KL{q\pp{\zr\vert\xr}}{p\pp{\zr\vert\xr}}.$
It is not hard to show that in the setting of \ref{thm:main-result}
that
\[
\SKL{q\pp{\zr,\xr}}{p\pp{\zr,\xr}}=\SKL{q\pp{\zr\vert\xr}}{p\pp{\zr\vert\xr}}.
\]

This is true because of cancellations between $q$ and $p$ due to
the inclusion of the $p\pp{\x}$ term in $q\pp{\x,\z}.$

\textbf{Use with randomized inference methods}. In practice, approximate
inference algorithms are often non-deterministic. This is not reflected
by the notation $q\pp{\z\vert\x}$. With non-determinism, no modification
to the diagnostic technique in \ref{alg:simple-diagnostic} is needed.
Only the interpretation is slightly different. To formalize what the
diagnostic measures in this case, define $q\pp{\z\vert\x,\bm{\omega}}$
to be the approximate posterior produced where $\bm{\omega}$ are
the random numbers underlying the algorithm. Then, a diagnostic can
be defined as the \emph{expected} divergence between $q$ and $p$,
i.e., $\E_{\bm{\r{\omega}}}\SKL{q\pp{\zr,\r x\vert\bm{\r{\omega}}}}{p\pp{\zr,\r x}},$
where $q\pp{\z,\x\vert\bm{\r{\omega}}}=p\pp{\x}q\pp{\z\vert\x,\bm{\r{\omega}}}.$
This is still a very reasonable measure of the accuracy of inference.
For simplicity, our presentation mostly neglects the issue of non-determinism
in approximating distributions.

\subsection{Related Work}

There are several lines of related work not mentioned so far. One
recent line of work explores inference diagnostics based on Stein's
method \citep{Gorham_2017_MeasuringSampleQuality,Gorham_2015_MeasuringSampleQuality}.
The idea is to create sets of functions whose true expectations must
be zero. Any deviation from zero in those functions indicates inference
failure. This diagnostic can be used with methods that are asymptotically
approximate. However, it is intended for cases when error decreases
to zero. There is no claim that the magnitude of the diagnostic is
a good measure of the usefulness of an approximate posterior.

The test proposed by \citet{Geweke_2004_GettingItRight} is an interesting
early diagnostic that repeatedly simulates datasets in a non-independent
manner. The idea is to iteratively sample $\x\sim p\pp{\x\vert\z}$,
the run inference to produce $q\pp{\z\vert\x}$ and then sample $\z\sim q\pp{\z\vert\x}.$
Then, one compares the expectation of some function $g\pp{\z,\x}$
to those on exact samples $\pp{\z,\x}\sim p\pp{\z,\x}.$ If $q$ is
exact, these expectations should match. We prefer an approach where
each simulation is independent since this is easier to parallelize,
avoids correlations between simulations, and makes it easier to compute
error measures like confidence intervals.

Bidirectional MCMC \citep{Grosse_2015_Sandwichingmarginallikelihood}
runs MCMC on repeated simulated datasets to get upper and lower bounds
on the marginal likelihood $\log p\pp{\x}.$ This is intended as a
technique to evaluate the quality of a model, not as a diagnostic
for inference. Still, in principle one could use these to transform
an ELBO into bounds on the KL-divergence. One drawback is the expense
of repeatedly running MCMC. Typically, variational inference is used
in settings where MCMC would be too expensive.

\subsection{Limitations and Future Work}

This work has several limitations shared with all simulation-based
diagnostics: First, computing them requires repeating inference numerous
times. This comes with an associated cost. Second, these methods can
be overly pessimistic when used with extremely broad or uninformative
priors. It is important that the model is chosen so that simulated
data are representative of the datasets one cares about. Third, the
diagnostic measures \emph{average accuracy over data simulated the
prior}, as opposed to the expected accuracy for a \emph{particular
dataset.} (Put another way, the diagnostic is arguably frequentist
rather than Bayesian.)

One might be concerned about the success of this diagnostic when used
with variational inference methods. Namely, VI typically minimizes
$\KL qp$ while the diagnostic is based on the symmetric divergence.
Informally, VI cares about finding a distribution that is close in
a ``mode finding'' divergence, while the diagnostic measures both
``mode finding'' and ``mode spanning''. It is possible that a
distribution could be close in VI's objective, yet yield a high diagnostic
value. This is arguably a flaw not of the diagnostic, but of variational
inference. One interesting future direction would be to investigate
recent VI variants that try to minimize other divergences \citep{Li__RenyiDivergenceVariational,Dieng_2017_VariationalInferencechi}.

In future work, it would be interesting to address MCMC methods. Of
course, most MCMC methods are not suitable for this framework. However,
some methods like annealed importance sampling \citep{Neal_1998_AnnealedImportanceSampling}
formally create augmented target and proposal densities at a variety
of ``temperatures''. It may be possible to use the diagnostic proposed
here to measure the symmetric divergence between these augmented distributions.
This could potentially offer a diagnostic for MCMC with the unusual
property that the diagnostic going to zero is both necessary and sufficient
to guarantee convergence to the stationary distribution.

\cleardoublepage{}

\bibliographystyle{plainnat}
\bibliography{justindomke_zotero_betterbibtex2}

\begin{thebibliography}{25}
\providecommand{\natexlab}[1]{#1}
\providecommand{\url}[1]{\texttt{#1}}
\expandafter\ifx\csname urlstyle\endcsname\relax
  \providecommand{\doi}[1]{doi: #1}\else
  \providecommand{\doi}{doi: \begingroup \urlstyle{rm}\Url}\fi

\bibitem[Agakov and Barber(2004)]{Agakov_2004_AuxiliaryVariationalMethod}
Felix~V. Agakov and David Barber.
\newblock An {{Auxiliary Variational Method}}.
\newblock In \emph{Neural {{Information Processing}}}, Lecture {{Notes}} in
  {{Computer Science}}, pages 561--566. {Springer, Berlin, Heidelberg}, 2004.

\bibitem[Bachman and Precup(2015)]{Bachman_2015_TrainingDeepGenerative}
Philip Bachman and Doina Precup.
\newblock Training {{Deep Generative Models}}: {{Variations}} on a {{Theme}}.
\newblock In \emph{{{NIPS Workshop}}: {{Advances}} in {{Approximate Bayesian
  Inference}}}, 2015.

\bibitem[Burda et~al.(2015)Burda, Grosse, and
  Salakhutdinov]{Burda_2015_ImportanceWeightedAutoencoders}
Yuri Burda, Roger Grosse, and Ruslan Salakhutdinov.
\newblock Importance {{Weighted Autoencoders}}.
\newblock In \emph{{{ICLR}}}, 2015.

\bibitem[Cook et~al.(2006)Cook, Gelman, and
  Rubin]{Cook_2006_ValidationSoftwareBayesian}
Samantha~R Cook, Andrew Gelman, and Donald~B Rubin.
\newblock Validation of {{Software}} for {{Bayesian Models Using Posterior
  Quantiles}}.
\newblock \emph{Journal of Computational and Graphical Statistics}, 15\penalty0
  (3):\penalty0 675--692, 2006.

\bibitem[Cover and Thomas(2006)]{Cover_2006_Elementsinformationtheory}
T.~M. Cover and Joy~A. Thomas.
\newblock \emph{Elements of Information Theory}.
\newblock {Wiley-Interscience}, {Hoboken, N.J}, 2nd ed edition, 2006.

\bibitem[Cremer et~al.(2017)Cremer, Morris, and
  Duvenaud]{Cremer_2017_ReinterpretingImportanceWeightedAutoencoders}
Chris Cremer, Quaid Morris, and David Duvenaud.
\newblock Reinterpreting {{Importance}}-{{Weighted Autoencoders}}.
\newblock \emph{arXiv:1704.02916 [stat]}, 2017.

\bibitem[Dieng et~al.(2017)Dieng, Tran, Ranganath, Paisley, and
  Blei]{Dieng_2017_VariationalInferencechi}
Adji~Bousso Dieng, Dustin Tran, Rajesh Ranganath, John Paisley, and David Blei.
\newblock Variational {{Inference}} via {$\chi$}-{{Upper Bound Minimization}}.
\newblock In \emph{{{NeurIPS}}}, 2017.

\bibitem[Domke and Sheldon(2018)]{Domke_2018_ImportanceWeightingVariational}
Justin Domke and Daniel Sheldon.
\newblock Importance {{Weighting}} and {{Variational Inference}}.
\newblock In \emph{{{NeurIPS}}}, 2018.

\bibitem[Gelman and Rubin(1992)]{Gelman_1992_InferenceIterativeSimulation}
Andrew Gelman and Donald~B. Rubin.
\newblock Inference from {{Iterative Simulation Using Multiple Sequences}}.
\newblock \emph{Statist. Sci.}, 7\penalty0 (4):\penalty0 457--472, 1992.

\bibitem[Geweke(2004)]{Geweke_2004_GettingItRight}
John Geweke.
\newblock Getting {{It Right}}: {{Joint Distribution Tests}} of {{Posterior
  Simulators}}.
\newblock \emph{Journal of the American Statistical Association}, 99\penalty0
  (467):\penalty0 799--804, 2004.

\bibitem[Gorham and Mackey(2015)]{Gorham_2015_MeasuringSampleQuality}
Jackson Gorham and Lester Mackey.
\newblock Measuring {{Sample Quality}} with {{Stein}}'s {{Method}}.
\newblock In \emph{{{NeurIPS}}}, 2015.

\bibitem[Gorham and Mackey(2017)]{Gorham_2017_MeasuringSampleQuality}
Jackson Gorham and Lester Mackey.
\newblock Measuring {{Sample Quality}} with {{Kernels}}.
\newblock In \emph{{{PMLR}}}, pages 1292--1301, 2017.

\bibitem[Grosse et~al.(2015)Grosse, Ghahramani, and
  Adams]{Grosse_2015_Sandwichingmarginallikelihood}
Roger~B. Grosse, Zoubin Ghahramani, and Ryan~P. Adams.
\newblock Sandwiching the marginal likelihood using bidirectional {{Monte
  Carlo}}.
\newblock \emph{arXiv:1511.02543 [cs, stat]}, 2015.

\bibitem[K{\'e}ry and Schaub(2012)]{Kery_2012_Bayesianpopulationanalysis}
Marc K{\'e}ry and Michael Schaub.
\newblock \emph{Bayesian Population Analysis Using {{WinBUGS}}: A Hierarchical
  Perspective}.
\newblock {Academic Press}, {Boston}, 1st ed edition, 2012.

\bibitem[Kucukelbir et~al.(2017)Kucukelbir, Tran, Ranganath, Gelman, and
  Blei]{Kucukelbir_2017_AutomaticDifferentiationVariational}
Alp Kucukelbir, Dustin Tran, Rajesh Ranganath, Andrew Gelman, and David~M.
  Blei.
\newblock Automatic {{Differentiation Variational Inference}}.
\newblock \emph{Journal of Machine Learning Research}, 18\penalty0
  (14):\penalty0 1--45, 2017.

\bibitem[Le et~al.(2018)Le, Igl, Rainforth, Jin, and
  Wood]{Le_2018_AutoEncodingSequentialMonte}
Tuan~Anh Le, Maximilian Igl, Tom Rainforth, Tom Jin, and Frank Wood.
\newblock Auto-{{Encoding Sequential Monte Carlo}}.
\newblock In \emph{{{ICLR}}}, 2018.

\bibitem[Li and Turner()]{Li__RenyiDivergenceVariational}
Yingzhen Li and Richard~E Turner.
\newblock R{\'e}nyi {{Divergence Variational Inference}}.
\newblock page~9.

\bibitem[Lunn(2013)]{Lunn_2013_BUGSbookpractical}
David Lunn.
\newblock \emph{The {{BUGS}} Book: A Practical Introduction to {{Bayesian}}
  Analysis}.
\newblock Texts in Statistical Science. {CRC Press, Taylor \& Francis Group},
  {Boca Raton, FL}, 2013.

\bibitem[Maddison et~al.(2017)Maddison, Lawson, Tucker, Heess, Norouzi, Mnih,
  Doucet, and Teh]{Maddison_2017_FilteringVariationalObjectives}
Chris~J Maddison, John Lawson, George Tucker, Nicolas Heess, Mohammad Norouzi,
  Andriy Mnih, Arnaud Doucet, and Yee Teh.
\newblock Filtering {{Variational Objectives}}.
\newblock In \emph{{{NeurIPS}}}, 2017.

\bibitem[Naesseth et~al.(2018)Naesseth, Linderman, Ranganath, and
  Blei]{Naesseth_2018_VariationalSequentialMonte}
Christian~A. Naesseth, Scott~W. Linderman, Rajesh Ranganath, and David~M. Blei.
\newblock Variational {{Sequential Monte Carlo}}.
\newblock In \emph{{{AISTATS}}}, volume~84 of \emph{Proceedings of {{Machine
  Learning Research}}}, pages 968--977. {PMLR}, 2018.

\bibitem[Neal(1998)]{Neal_1998_AnnealedImportanceSampling}
Radford~M. Neal.
\newblock Annealed {{Importance Sampling}}.
\newblock \emph{arXiv:physics/9803008}, 1998.

\bibitem[Owen(2013)]{Owen_2013_MonteCarlotheory}
Art Owen.
\newblock \emph{Monte {{Carlo}} Theory, Methods and Examples}.
\newblock 2013.

\bibitem[Roeder et~al.(2017)Roeder, Wu, and
  Duvenaud]{Roeder_2017_StickingLandingSimple}
Geoffrey Roeder, Yuhuai Wu, and David~K Duvenaud.
\newblock Sticking the {{Landing}}: {{Simple}}, {{Lower}}-{{Variance Gradient
  Estimators}} for {{Variational Inference}}.
\newblock In \emph{{{NeurIPS}}}, 2017.

\bibitem[Vehtari et~al.(2020)Vehtari, Gelman, Simpson, Carpenter, and
  B{\"u}rkner]{Vehtari_2020_Ranknormalizationfoldinglocalization}
Aki Vehtari, Andrew Gelman, Daniel Simpson, Bob Carpenter, and Paul-Christian
  B{\"u}rkner.
\newblock Rank-normalization, folding, and localization: {{An}} improved
  \$\textbackslash{}widehat\{\vphantom\}{{R}}\vphantom\{\}\$ for assessing
  convergence of {{MCMC}}.
\newblock \emph{arXiv:1903.08008 [stat]}, 2020.

\bibitem[Yao et~al.(2018)Yao, Vehtari, Simpson, and Gelman]{Yao_2018_YesDidIt}
Yuling Yao, Aki Vehtari, Daniel Simpson, and Andrew Gelman.
\newblock Yes, but {{Did It Work}}?: {{Evaluating Variational Inference}}.
\newblock In \emph{{{ICML}}}, 2018.

\end{thebibliography}
\cleardoublepage\onecolumn

\section{Models\label{sec:Models}}

\textbf{GLM Binomial}. This is a model of the number of peregrine
pairs $c_{i}$ in the French Jura in year $x_{i}.$ The data is from
between $1964$ and $2003,$ but $x_{i}$ is scaled to between $-1$
and $+1$. The model and data are from \citet[Sec. 3.5]{Kery_2012_Bayesianpopulationanalysis}.

\begin{align*}
\r{\alpha} & \sim\mathcal{N}\pp{0,\pp{10}^{2}}\\
\r{\beta}_{1} & \sim\mathcal{N}\pp{0,\pp{10}^{2}}\\
\r{\beta}_{2} & \sim\mathcal{N}\pp{0,\pp{10}^{2}}\\
\r c_{i} & \sim\mathrm{Binomial}\pp{n_{i},\alpha+\beta_{1}x_{i}+\beta_{2}x_{i}^{2}}
\end{align*}

\textbf{Heart Transplants}. This is a model of a hypothetical population
of  patients who underwent a surgery, of whom $\r y_{T}=8$ survived
\citep[Ex. 3.5.1]{Lunn_2013_BUGSbookpractical}. These were tracked
to see the number of years $\r s_{i}$ ($i\in\left\{ 1,\cdots,8\right\} $)
that the $i$-th patient who survived surgery lived post-surgery.
This is assumed to be determined by an exponential distribution with
parameter $\theta$. Thus, the model is:

\begin{align*}
\r p_{T} & \sim\mathrm{Uniform}\pp{0,1}\\
\r y_{T} & \sim\mathrm{Binomial}\pp{N,\r p_{T}}\\
\r{\theta} & \sim\mathrm{Gamma}\pars{1/3,1/3}\\
\r s_{i} & \sim\mathrm{Exponential}\pp{\theta}
\end{align*}

Note that, when generating synthetic datasets for this model, we always
use the same set of variables $\r s_{1},\cdots,\r s_{8}$, independent
of the value of $\r y_{T}$. This is done because of the difficulties
posed by having different dimensionality in different realizations
of the posterior. While not fully in keeping with the spirit of the
original model, this still defines a perfectly valid probabilistic
model and test of the diagnostic.

\textbf{Hospitals}. This is a  hierarchical model of the mortality
rate of  English hospitals performing heart surgery \citep[Ex. 10.1.1]{Lunn_2013_BUGSbookpractical}.
The data is $\left\{ \pp{n_{i},y_{i}}\right\} $ where $n_{i}$ is
the number of operations in hospital $i$ and $y_{i}$ is the corresponding
number of deaths. The logit of the true mortality rate $\theta_{i}$
of hospital $i$ is a Gaussian with unknown mean $\mu$ and standard
deviation $\omega$. The latent variables are $\omega$, $\mu$, and
$\left\{ \theta_{i}\right\} $.  \vspace{-0.6cm}

\begin{align*}
\r{\omega} & \sim\mathrm{Uniform}\pars{.25,1}\\
\r{\mu} & \sim\mathrm{Uniform}\pars{-3,3}\\
\mathrm{Logit}\pp{\theta_{i}} & \sim\mathcal{N}\pars{\mu,\r{\omega}^{2}}\\
y_{i} & \sim\mathrm{Binomial}\pp{n_{i},\theta_{i}}
\end{align*}

The original model has a very wide prior on $\omega$ and $\r{\mu}$,
which leads to the problems discussed in \ref{subsec:new-diagnostic-discussion}.
We use the above model with more modest priors.

\textbf{Ionosphere}. This is a classic dataset for binary classification.
We model it as a Bayesian logistic regression problem with a standard
Gaussian prior over the weights $\r w.$

\textbf{Concrete}. This is a well-known dataset for linear regression.
We model it as a Bayesian linear regression problem with a standard
Gaussian prior over the weights $\r w.$ This model is particularly
notable because the true posterior is exactly Gaussian. Since both
Laplace's method and variational inference can exactly represent such
a posterior, this provides an important test if the diagnostic can
correctly recognize inference success when it occurs.

\cleardoublepage{}

\section{Theory}

\mainresult*
\begin{proof}
The divergence $\SKL{q\pp{\zr,\xr}}{p\pp{\zr,\xr}}$ is equal to

\begin{alignat*}{1}
\E_{p\pp{\zr,\xr}}\log\frac{p\pp{\zr,\xr}}{q\pp{\zr,\xr}}+\E_{q\pp{\zr,\xr}}\log\frac{q\pp{\zr,\xr}}{p\pp{\zr,\xr}} & =\E_{p\pp{\xr}}\bracs{\E_{p\pp{\zr\vert\xr}}\log\frac{p\pp{\zr,\xr}}{q\pp{\zr,\xr}}+\E_{q\pp{\zr\vert\xr}}\log\frac{q\pp{\zr,\xr}}{p\pp{\zr,\xr}}}\\
 & =\E_{p\pp{\xr}}\bracs{\E_{p\pp{\zr\vert\xr}}\log\frac{p\pp{\zr,\xr}}{q\pp{\zr\vert\xr}}+\E_{q\pp{\zr\vert\xr}}\log\frac{q\pp{\zr\vert\xr}}{p\pp{\zr,\xr}}}
\end{alignat*}

In the first line we use the fact that $q\pp{\x}=p\pp{\x},$ while
in the second line we pull out a factor of $\log q\pp{\x}$ from each
term, which cancel. The claimed result is the same as the last line
with a sign change.
\end{proof}

\mainIWVI*
\begin{proof}
Start with the result of \ref{cor:main-result-hidden}. 
\begin{alignat*}{1}
\SKL{q\pp{\r z,\r h,\r x}}{p\pp{\r z,\r h,\r x}} & =\E\log\frac{p\pp{\r z,\r h,\r x}}{q\pp{\r z,\r h|\r x}}-\log\frac{p\pp{\tilde{\r z},\tilde{\r h},\r x}}{q\pp{\tilde{\r z},\tilde{\r h}|\r x}}\\
\pp{\r z,\r x} & \sim p\pp{\z,\x}\\
\r h & \sim p\pp{\h|\z,\x}\\
\pp{\tilde{\r z},\tilde{\r h}} & \sim q\pp{\z,\h|\x}.
\end{alignat*}
Now, make the following transformations
\begin{eqnarray*}
q & \Rightarrow & q_{IW}\\
p & \Rightarrow & p_{IW}\\
\z & \Rightarrow & \z_{1}\\
\h & \Rightarrow & \pp{\z_{2},\cdots,\z_{M}}.
\end{eqnarray*}
Then, we get
\begin{alignat*}{1}
\SKL{q_{IW}\pp{\r z_{1}\cdots,\zr_{M},\r x}}{p\pp{\r z_{1},\cdots,\zr_{M},\r x}} & =\E\log\frac{p_{IW}\pp{\r z_{1}\cdots,\zr_{M},\r x}}{q_{IW}\pp{\r z_{1}\cdots,\zr_{M}|\r x}}-\log\frac{p_{IW}\pp{\tilde{\r z}_{1}\cdots,\tilde{\zr}_{M},\r x}}{q_{IW}\pp{\tilde{\r z}_{1}\cdots,\tilde{\zr}_{M}|\r x}}\\
\pp{\r z_{1},\r x} & \sim p_{IW}\pp{\z_{1},\x}\\
\r z_{2}\cdots,\zr_{M} & \sim p_{IW}\pp{\z_{2}\cdots,\z_{M}|\z,\x}\\
\pp{\tilde{\zr}_{1},\cdots,\tilde{\zr}_{M}} & \sim q_{IW}\pp{\z_{1}\cdots,\z_{M}|\x}.
\end{alignat*}

Now, note that
\begin{eqnarray*}
\frac{p_{IW}\pp{\r z_{1}\cdots,\zr_{M},\r x}}{q_{IW}\pp{\r z_{1}\cdots,\zr_{M}|\r x}} & = & \frac{1}{M}\sum_{m=1}^{M}\frac{p\pp{\z_{m},\x}}{q\pp{\z_{m}\vert\x}}\\
\pp{\r z_{1},\r x} & \sim & p_{IW}\pp{\z_{1},\x}\\
 & = & p\pp{\z_{1},\x}\\
\r z_{2}\cdots,\zr_{M} & \sim & p_{IW}\pp{\z_{2}\cdots,\z_{M}|\z,\x}\\
 & = & \prod_{m=1}^{M}q\pp{\z_{m}\vert\x}
\end{eqnarray*}
This leaves us with the result of
\begin{alignat*}{1}
 & \SKL{q_{IW}\pp{\r z_{1},\cdots,\zr_{M},\r x}}{p_{IW}\pp{\r z_{1},\cdots,\r x}}\\
 & =\E\log\frac{1}{M}\sum_{m=1}^{M}\frac{p\pp{\z_{m},\x}}{q\pp{\z_{m}\vert\x}}-\log\frac{1}{M}\sum_{m=1}^{M}\frac{p\pp{\tilde{\z}_{m},\x}}{q\pp{\tilde{\z}_{m}\vert\x}},\\
 & =\E\log\sum_{m=1}^{M}\frac{p\pp{\z_{m},\x}}{q\pp{\z_{m}\vert\x}}-\log\sum_{m=1}^{M}\frac{p\pp{\tilde{\z}_{m},\x}}{q\pp{\tilde{\z}_{m}\vert\x}},
\end{alignat*}
\begin{alignat*}{1}
\pp{\r z_{1},\r x} & \sim p\pp{\z,\x}\\
\zr_{2},\cdots,\zr_{M} & \sim q\pp{\z\vert\x}\\
\pp{\tilde{\zr}_{1},\cdots,\tilde{\zr}_{M}} & \sim q_{IW}\pp{\z_{1}\cdots,\z_{M}|\x}.
\end{alignat*}

Now, finally, note that the expectation is unchanged under permutations
of the order of $\tilde{\zr}_{1},\cdots,\tilde{\zr}_{M}$. Thus, the
expectation is unchanged if we replace the distribution with
\[
\tilde{\zr}_{1},\cdots,\tilde{\zr}_{M}\sim q\pp{\z\vert\x}.
\]
\end{proof}

\end{document}